\documentclass[Afour,sageh,times]{sagej}

\usepackage[colorlinks,bookmarksopen,bookmarksnumbered,citecolor=red,urlcolor=red]{hyperref}
\usepackage{utils}
\usepackage{mathdots}
\usepackage{subcaption}

\graphicspath{{figures/}}
\setenumerate{leftmargin=*}
\setitemize{leftmargin=*}

\begin{document}

\runninghead{Richards, Azizan, Slotine, and Pavone}

\title{Control-oriented meta-learning}

\author{%
    Spencer M.~Richards\affilnum{1},
    Navid Azizan\affilnum{2},
    Jean-Jacques Slotine\affilnum{2}, and
    Marco Pavone\affilnum{1}
}

\affiliation{%
    \affilnum{1}Department of Aeronautics \& Astronautics, Stanford University, Stanford, CA, U.S.A.\\
    \affilnum{2}Department of Mechanical Engineering, Massachusetts Institute of Technology, Cambridge, MA, U.S.A.
}
\corrauth{Spencer M.~Richards, Department of Aeronautics \& Astronautics, Stanford University, 496 Lomita Mall, Stanford, CA 94305, U.S.A.}
\email{spenrich@stanford.edu}

\begin{abstract}
    Real-time adaptation is imperative to the control of robots operating in complex, dynamic environments. Adaptive control laws can endow even nonlinear systems with good trajectory tracking performance, provided that any uncertain dynamics terms are linearly parameterizable with known nonlinear features. However, it is often difficult to specify such features a priori, such as for aerodynamic disturbances on rotorcraft or interaction forces between a manipulator arm and various objects. In this paper, we turn to data-driven modeling with neural networks to learn, offline from past data, an adaptive controller with an internal parametric model of these nonlinear features. Our key insight is that we can better prepare the controller for deployment with control-oriented meta-learning of features in closed-loop simulation, rather than regression-oriented meta-learning of features to fit input-output data. Specifically, we meta-learn the adaptive controller with closed-loop tracking simulation as the base-learner and the average tracking error as the meta-objective. With both fully-actuated and underactuated nonlinear planar rotorcraft subject to wind, we demonstrate that our adaptive controller outperforms other controllers trained with regression-oriented meta-learning when deployed in closed-loop for trajectory tracking control.
\end{abstract}


\maketitle

\section{Introduction}\label{sec:introduction}

\begin{figure}[t]
    \centering
    \includegraphics[width=\columnwidth]{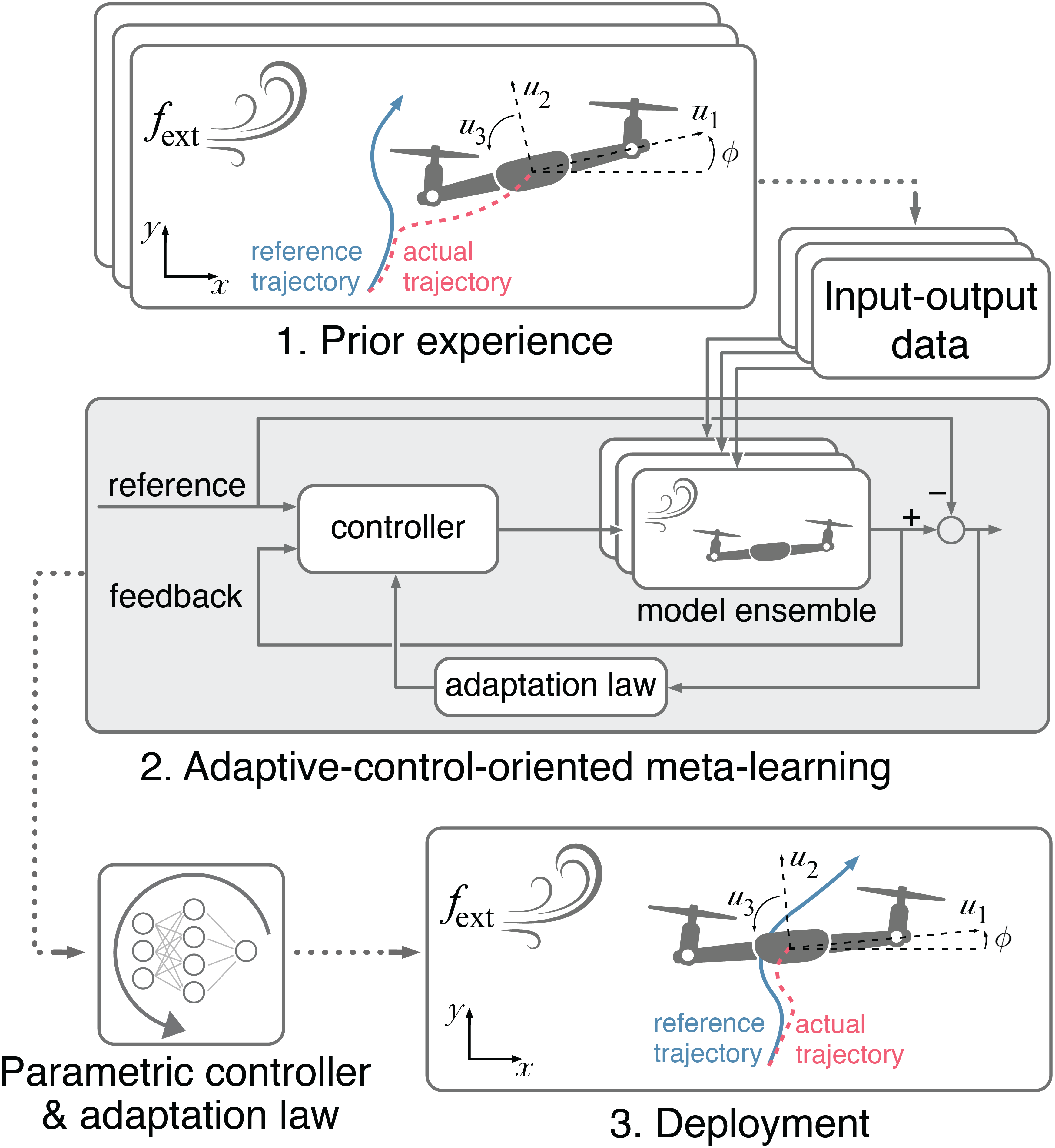}
    \caption{%
        While roboticists can often derive a model for how control inputs affect the system state, it is much more difficult to model prevalent external forces (e.g., from aerodynamics, contact, and friction) that adversely affect tracking performance. In this work, we present a method to meta-learn an adaptive controller offline from previously collected data. Our meta-learning is \emph{control-oriented} rather than regression-oriented; specifically, we: 1) collect input-output data on the true system, 2) train a parametric adaptive controller in closed-loop simulation to adapt well to each model of an ensemble constructed from past input-output data, and 3) test our adaptive controller on the real system. Our method contextualizes training within the downstream control objective, thereby engendering good tracking results at test time, which we demonstrate on fully-actuated and underactuated versions of a planar quadrotor system subject to wind.
    }\label{fig:pfar}
\end{figure}

Performant control in robotics is impeded by the complexity of the dynamical system consisting of the robot itself (i.e., its nonlinear equations of motion) and the interactions with its environment. Roboticists can often derive a physics-based robot model, and then choose from a suite of nonlinear control laws that each offer desirable control-theoretic properties (e.g., good tracking performance) in known, simple environments. Even in the face of model uncertainty, nonlinear control can still yield such properties with the help of \emph{real-time adaptation} to \emph{online} measurements, provided the uncertainty enters the system in a known, structured manner.

However, when a robot is deployed in complex scenarios, it is generally intractable to know even the structure of all possible configurations and interactions that the robot may experience. To address this, system identification and data-driven control seek to learn an accurate input-output model from past measurements. Recent years have also seen a dramatic proliferation of research in machine learning for control by leveraging powerful approximation architectures to predict and optimize the behaviour of dynamical systems. In general, such rich models require extensive data and computation to back-propagate gradients for many layers of parameters, and thus usually cannot be used in fast nonlinear control loops.

Moreover, machine learning of dynamical system models often prioritizes fitting input-output data, i.e., it is \emph{regression-oriented}, with the rationale that designing a controller for a highly accurate model engenders better closed-loop performance on the real system. However, decades of work in system identification and adaptive control recognize that, since a model is often learned for the purpose of control, the learning process itself should be tailored to the \emph{downstream control objective}. This concept of \emph{control-oriented} learning is exemplified by fundamental results in adaptive control theory; guarantees on tracking convergence can be attained \emph{without} convergence of the parameter estimates to those of the true system. 

\subsection{Contributions} 
In this work, we acknowledge this distinction between regression-oriented and control-oriented learning, and propose a control-oriented method to learn a parametric adaptive controller that performs well in closed-loop at test time. Critically, our method (outlined in \cref{fig:pfar}) focuses on \emph{offline} learning from past trajectory data. We formalize training the adaptive controller as a semi-supervised, bi-level meta-learning problem, with the average integrated tracking error across chosen target trajectories as the meta-objective. We use a closed-loop simulation with our adaptive controller as a base-learner, which we then back-propagate gradients through. We discuss how our formulation can be applied to adaptive controllers for general dynamical systems, then specialize it to different classes of nonlinear systems. Through our experiments, we show that by injecting the downstream control objective into offline meta-learning of an adaptive controller, we improve closed-loop trajectory tracking performance at test time in the presence of widely varying disturbances. We provide code to reproduce our results at \url{https://github.com/StanfordASL/Adaptive-Control-Oriented-Meta-Learning}.

A preliminary version of this article was presented at Robotics: Science and Systems 2021 \citep{RichardsAzizanEtAl2021}. In this revised and extended version, we additionally contribute: 1) exposition on how our meta-learning framework can be applied to control-affine and underactuated nonlinear systems, 2) analysis of the relationship between stability certificate functions and online adaptation laws for stable concurrent learning and control, 3) new experiments on an underactuated aerial vehicle system that reinforce our previous observations for fully-actuated systems, and 4) additional experiments with unknown time-varying disturbances.

\section{Related Work}

In this section, we review three key areas of work related to this paper: control-oriented system identification, adaptive control, and meta-learning.

\subsection{Control-Oriented System Identification}
Learning a system model for the express purpose of closed-loop control has been a hallmark of linear system identification since at least the early 1970s \citep{AstromWittenmark1971}. Due to the sheer amount of literature in this area, we direct readers to the book by \citet{Ljung1999} and the survey by \citet{Gevers2005}. Some salient works are the demonstrations by \citet{Skelton1989} on how large open-loop modelling errors do not necessarily cause poor closed-loop prediction, and the theory and practice from \citet{HjalmarssonGeversEtAl1996} and \citet{ForssellLjung2000} for iterative online closed-loop experiments that encourage convergence to a model with optimal closed-loop behaviour. In this paper, we focus on \emph{offline meta-learning} targeting a downstream closed-loop control objective, to train adaptive controllers for \emph{nonlinear} systems.

In nonlinear system identification, there is an emerging body of literature on data-driven, constrained learning for dynamical systems that encourages learned models and controllers to perform well in closed-loop. \citet{Khansari-ZadehBillard2011} and \citet{MedinaBillard2017} train controllers to imitate known invertible dynamical systems while constraining the closed-loop system to be stable. \citet{ChangRoohiEtAl2019} and \citet{SunJhaEtAl2020} jointly learn a controller and a stability certificate for known dynamics to encourage good performance in the resulting closed-loop system. \citet{SinghRichardsEtAl2020} jointly learn a dynamics model and a stabilizability certificate that regularizes the model to perform well in closed-loop, even with a controller designed a posteriori. Overall, these works concern learning a fixed model-controller pair. Instead, with offline meta-learning, we train an \emph{adaptive} controller that can update its internal representation of the dynamics online. We discuss future work explicitly incorporating stability constraints in~\cref{sec:conclusion}.

\subsection{Adaptive Control}
Broadly speaking, adaptive control concerns parametric controllers paired with an \emph{adaptation law} that dictates how the parameters are adjusted online in response to signals in a dynamical system \citep{SlotineLi1991,NarendraAnnaswamy2005,LandauLozanoEtAl2011,IoannouSun2012}. Since at least the 1950s, researchers in adaptive control have focused on parameter adaptation prioritizing control performance over parameter identification \citep{AseltineManciniEtAl1958}. Indeed, one of the oldest adaptation laws, the so-called MIT rule, is essentially gradient descent on the integrated squared \emph{tracking} error \citep{MareelsAndersonEtAl1987}. Tracking convergence to a reference signal is the primary result in Lyapunov stability analyses of adaptive control designs \citep{NarendraValavani1978,NarendraValavani1980}, with parameter convergence as a secondary result if the reference is persistently exciting \citep{AndersonJohnson1982,BoydSastry1986}. In the absence of persistent excitation, \citet{BoffiSlotine2021} show certain adaptive controllers also ``implicitly regularize'' \citep{AzizanHassibi2019,AzizanLaleEtAl2021} the learned parameters to have small Euclidean norm; moreover, different forms of implicit regularization (\eg sparsity-promoting) can be achieved by certain modifications of the adaptation law. Overall, adaptive control prioritizes control performance while learning parameters on a ``need-to-know'' basis, which is a principle that can be extended to many learning-based control contexts \citep{WensingSlotine2020}.

Stable adaptive control of nonlinear systems often relies on linearly parameterizable dynamics with known nonlinear basis functions, \ie \emph{features}, and the ability to cancel these nonlinearities stably with the control input when the parameters are known exactly \citep{SlotineLi1987,SlotineLi1989,SlotineLi1991,LopezSlotine2021}. When such features cannot be derived a priori, function approximators such as neural networks \citep{SannerSlotine1992,JoshiChowdhary2019,JoshiVirdiEtAl2020}, Gaussian processes \citep{GrandeChowdharyEtAl2013,GahlawatZhaoEtAl2020}, and random Fourier features \citep{BoffiTuEtAl2021} can be used and updated online in the adaptive control loop. However, \emph{fast} closed-loop adaptive control with complex function approximators is hindered by the computational effort required to train them; this issue is exacerbated by the practical need for controller gain tuning. In our paper, we focus on offline meta-training of neural network features and controller gains from collected data, with controller structures that can operate in fast closed-loops.

\subsection{Meta-Learning}\label{sec:literature-meta-learning}
Meta-learning is the tool we use to inject the downstream adaptive control application into offline learning from data. Informally, meta-learning or ``learning to learn'' improves knowledge of how to best optimize a given meta-objective across \emph{different} tasks. In the literature, meta-learning has been formalized in various manners; we refer readers to \citet{HospedalesAntoniouEtAl2021} for a survey of them. In general, the algorithm chosen to solve a specific task is the \emph{base-learner}, while the algorithm used to optimize the meta-objective is the \emph{meta-learner}. In our work, when trying to make a dynamical system track several target trajectories, each trajectory is associated with a ``task'', the adaptive tracking controller is the base-learner, and the average tracking error across all of these trajectories is the meta-objective we want to minimize.

Many works try to meta-learn a dynamics model offline that can best fit new input-output data gathered during a particular task. That is, the base- and meta-learners are \emph{regression-oriented}. \citet{BertinettoHenriquesEtAl2019} and \citet{LeeMajiEtAl2019} back-propagate through closed-form ridge regression solutions for few-shot learning, with a maximum likelihood meta-objective. \citet{OConnellShiEtAl2021} apply this same method to learn neural network features for nonlinear mechanical systems. \citet{HarrisonSharmaEtAl2018,HarrisonSharmaEtAl2018b} more generally back-propagate through a Bayesian regression solution to train a Bayesian prior dynamics model with nonlinear features. \citet{NagabandiClaveraEtAl2019} use a maximum likelihood meta-objective, and gradient descent on a multi-step likelihood objective as the base-learner. \citet{BelkhaleLiEtAl2021} also use a maximum likelihood meta-objective, albeit with the base-learner as a maximization of the Evidence Lower BOund (ELBO) over parameterized, task-specific variational posteriors; at test time, they perform latent variable inference online in a slow control loop.

\citet{FinnAbbeelEtAl2017,RajeswaranGhotraEtAl2017}, and \citet{ClaveraRothfussEtAl2018} meta-train a policy with the expected accumulated reward as the meta-objective, and a policy gradient step as the base-learner. These works are similar to ours in that they infuse offline learning with a more control-oriented flavour. However, while policy gradient methods are amenable to purely data-driven models, they beget slow control-loops due to the sampling and gradients required for each update. Instead, we back-propagate gradients through \emph{offline} closed-loop simulations to train adaptive controllers designed for fast online implementation. This yields a meta-trained adaptive controller that enjoys the performance of principled design inspired by the rich body of control-theoretical literature.

\section{Problem Statement}\label{sec:problem-statement}

In this paper, we are interested in controlling the continuous-time, nonlinear dynamical system
\begin{equation}\label{eq:nonlinear-system}
    \dot{x} = f(x, u, w),
\end{equation}
where $x(t) \in \R^n$ is the state, $u(t) \in \R^m$ is the control input, and $w(t) \in \R^d$ is some unknown disturbance, each at time $t \in \R_{\geq 0}$. Specifically, for a given target trajectory $x^*(t) \in \R^n$, we want to choose~$u(t)$ such that~$x(t)$ converges to~$x^*(t)$; we then say $u(t)$ makes the system~\cref{eq:nonlinear-system} track $x^*(t)$. 

Since $w(t)$ is unknown and possibly time-varying, we want to design a feedback law $u = \pi(x,x^*,\hat{a})$ with parameters $\hat{a}(t)$ that are updated \emph{online} according to a chosen \emph{adaptation law} $\dot{\hat{a}} = \rho(x,x^*,\hat{a})$. We refer to $(\pi,\rho)$ together as an \emph{adaptive controller}. For example, consider the control-affine system
\begin{equation}\label{eq:matched-uncertainty}
    \dot{x} = f(x) + B(x)(u + Y(x)a),
\end{equation}
where $f$, $B$, and $Y$ are known, possibly nonlinear maps, and~$a$ is a vector of \emph{unknown} parameters. We can interpret ${ w(t) = Y(x(t))a }$ as the disturbance in this system. A reasonable feedback law choice would be
\begin{equation}\label{eq:adaptive-controller}
    u = \bar{\pi}(x,x^*) - Y(x)\hat{a},
\end{equation}
where $\bar{\pi}$ ensures ${ \dot{x} = f(x) + B(x)\bar{\pi}(x,x^*) }$ tracks $x^*(t)$, and the term $Y(x)\hat{a}$ is meant to cancel $Y(x)a$ in~\cref{eq:matched-uncertainty}. For this reason, $Y(x)a$ is termed a \emph{matched uncertainty} in the literature. If the adaptation law $\dot{\hat{a}} = \rho(x,x^*,\hat{a})$ is designed such that~$Y(x(t))\hat{a}(t)$ converges to~$Y(x(t))a(t)$, then we can use~\cref{eq:adaptive-controller} to make~\cref{eq:matched-uncertainty} track~$x^*(t)$. Critically, this is \emph{not} the same as requiring~$\hat{a}(t)$ to converge to~$a(t)$. Since~$Y(x)a$ depends on~$x(t)$ and hence indirectly on the target~$x^*(t)$, the roles of feedback and adaptation are inextricably linked by the tracking control objective. Overall, learning in adaptive control is done on a ``need-to-know'' basis to cancel~$Y(x)a$ in \emph{closed-loop}, rather than to estimate~$a$ in open-loop.

\section{Bi-Level Meta-Learning}\label{sec:meta-learning}

We now describe some preliminaries on meta-learning akin to \citet{FinnAbbeelEtAl2017} and \citet{RajeswaranFinnEtAl2019}, so that we can apply these ideas in the next section to the adaptive control problem~\cref{eq:nonlinear-system} and in \cref{sec:baselines} to our baselines.

In machine learning, we typically seek some optimal parameters ${\xi^* \in \argmin_\xi \ell(\xi,\mathcal{D})}$, where~$\ell$ is a scalar-valued loss function and~$\mathcal{D}$ is some data set. In meta-learning, we instead have a collection of loss functions~$\{\ell_i\}_{i=1}^M$, training data sets~$\{\mathcal{D}^{\mathrm{train}}_i\}_{i=1}^M$, and evaluation data sets $\{\mathcal{D}^{\mathrm{eval}}_i\}_{i=1}^M$, where each~$i$ corresponds to a task. Moreover, during each task~$i$, we can apply an adaptation mechanism ${\operatorname{\mathrm{Adapt}} : (\theta, \mathcal{D}^{\mathrm{train}}_i) \mapsto \xi_i}$ to map so-called meta-parameters~$\theta$ and the task-specific training data~$\mathcal{D}^{\mathrm{train}}_i$ to task-specific parameters~$\xi_i$. The crux of meta-learning is to solve the bi-level problem
\begin{equation}\label{eq:meta-learning}
\begin{aligned}
    \theta^* \in\, &\argmin_\theta \frac{1}{M}\rbr*{
        \sum_{i=1}^M \ell_i(\xi_i, \mathcal{D}_i^{\mathrm{eval}})
        + \mu_\mathrm{meta}\norm{\theta}_2^2   
    }
    \\&\ \mathrm{s.t.}\enspace
    \xi_i = \operatorname{\mathrm{Adapt}}(\theta,\mathcal{D}^{\mathrm{train}}_i)
\end{aligned},
\end{equation}
with regularization coefficient $\mu_\mathrm{meta} \geq 0$, thereby producing meta-parameters~$\theta^*$ that are on average well-suited to being adapted for each task. This motivates the moniker ``learning to learn'' for meta-learning. The optimization~\cref{eq:meta-learning} is the meta-problem, while the average loss is the meta-loss. The adaptation mechanism $\mathrm{Adapt}$ is termed the \emph{base-learner}, while the algorithm used to solve~\cref{eq:meta-learning} is termed the \emph{meta-learner} \citep{HospedalesAntoniouEtAl2021}.

Generally, the meta-learner is chosen to be some gradient descent algorithm. Choosing a good base-learner is an open problem in meta-learning research. \citet{FinnAbbeelEtAl2017} propose using a gradient descent step as the base-learner, such that $\xi_i = \theta - \alpha\grad_\theta\ell_i(\theta,\mathcal{D}_i^\mathrm{train})$ in~\cref{eq:meta-learning} with some learning rate~${\alpha > 0}$. This approach is general in that it can be applied to any differentiable task loss functions. \citet{BertinettoHenriquesEtAl2019} and \citet{LeeMajiEtAl2019} instead study when the base-learner can be expressed as a convex program with a differentiable closed-form solution. In particular, they consider ridge regression with the hypothesis $\hat{y} = A g(x; \theta)$, where~$A$ is a matrix and $g(x; \theta)$ is some vector of nonlinear features parameterized by~$\theta$. For the base-learner, they use
\begin{equation}
    \xi_i = \argmin_A \sum_{(x,y) \in \mathcal{D}_i^\mathrm{train}} \norm{y - A g(x; \theta)}_2^2 + \mu_\mathrm{ridge}\norm{A}_F^2,
\end{equation}
with regularization coefficient $\mu_\mathrm{ridge} > 0$ for the Frobenius norm $\norm{A}_F^2$, which admits a differentiable, closed-form solution. Instead of adapting~$\theta$ to each task~$i$ with a single gradient step, this approach leverages the convexity of ridge regression tasks to minimize the task loss analytically.

\section{Adaptive Control as a Base-Learner}\label{sec:method}

We now present the key idea of our paper, which uses meta-learning concepts introduced in \cref{sec:meta-learning} to tackle the problem of learning to control~\cref{eq:nonlinear-system}. For the moment, we assume we can simulate the dynamics function~$f$ in~\cref{eq:nonlinear-system} offline and that we have~$M$ samples $\{w_j(t)\}_{j=1}^M$ for $t \in [0,T]$ in~\cref{eq:nonlinear-system}; we will eliminate these assumptions in~\cref{sec:model-ensemble}. 

\subsection{Meta-Learning from Feedback and Adaptation}
\label{sec:meta-adaptive-control}
In meta-learning vernacular, we treat a target trajectory~$x^*_i(t) \in \R^n$ and disturbance signal $w_j(t) \in \R^d$ together over some time horizon~$T > 0$ as the training data ${\mathcal{D}_{ij}^\mathrm{train} = \{x^*_i(t),w_j(t)\}_{t \in [0,T]}}$ for task~$(i,j)$. We wish to learn the static, possibly shared parameters $\theta$ of an adaptive controller
\begin{equation}\begin{aligned}
    u &= \pi(x,x^*,\hat{a}; \theta) \\
    \dot{\hat{a}} &= \rho(x,x^*,\hat{a}; \theta)
\end{aligned}~,
\end{equation}
such that $(\pi,\rho)$ engenders good tracking of $x^*_i(t)$ for ${t \in [0,T]}$ subject to the disturbance~$w_j(t)$. Our adaptation mechanism is the forward-simulation of our closed-loop system, \ie in~\cref{eq:meta-learning} we have $\xi_{ij} = \{x_{ij}(t),\hat{a}_{ij}(t),u_{ij}(t)\}_{t\in[0,T]}$, where
\begin{equation}\label{eq:simulate}
\begin{aligned}
    x_{ij}(t) &= x_{ij}(0) + \int_0^T f(x_{ij}(t),u_{ij}(t),w_j(t))\,dt, \\
    \hat{a}_{ij}(t) &= \hat{a}_{ij}(0) + \int_0^T \rho(x_{ij}(t),x^*_i(t),\hat{a}_{ij}(t); \theta)\,dt, \\
    u_{ij}(t) &= \pi(x_{ij}(t),x^*_i(t),\hat{a}_{ij}(t); \theta),
\end{aligned}\end{equation}
which we can compute with any Ordinary Differential Equation (ODE) solver. For simplicity, we always set $x_{ij}(0) = x^*_i(0)$ and $a_{ij}(0) = 0$. Our task loss is simply the average tracking error for the same target-disturbance pair, \ie $\mathcal{D}_{ij}^\mathrm{eval} = \{x^*_i(t),w_j(t)\}_{t \in [0,T]}$ and
\begin{equation}\label{eq:tracking-loss}
    \ell_{ij}(\xi_{ij},\mathcal{D}_{ij}^\mathrm{eval}) 
    = \frac{1}{T} \int_0^{T} \rbr*{%
        \norm{x_{ij}(t) - x^*_i(t)}_2^2 + \mu_\mathrm{ctrl}\norm{u_{ij}(t)}_2^2
    }\,dt,
\end{equation}
where $\mu_\mathrm{ctrl} \geq 0$ regularizes the average control effort $\frac{1}{T}\int_0^T\norm{u_{ij}(t)}_2^2\,dt$. This loss is inspired by the Linear Quadratic Regulator (LQR) from optimal control, and can be generalized to weighted norms. Suppose we construct~$N$ target trajectories~$\{x^*_i(t)\}_{i=1}^N$ and sample~$M$ disturbance signals~$\{w_j(t)\}_{j=1}^M$, thereby creating~$NM$ tasks. Combining~\cref{eq:simulate} and~\cref{eq:tracking-loss} for all $(i,j)$ in the form of \cref{eq:meta-learning} then yields the meta-problem
\begin{equation}\label{eq:meta-adaptive-control}
\begin{aligned}
    \min_\theta \
    &\frac{1}{NMT} \rbr*{ \sum_{i=1}^N \sum_{j=1}^M \int_0^{T} c_{ij}(t) \,dt + \mu_\mathrm{meta}\norm{\theta}_2^2 }
    \\&\mathrm{s.t.}\enspace\begin{aligned}[t]
        c_{ij} &= \norm{x_{ij} - x^*_i}_2^2 + \mu_\mathrm{ctrl}\norm{u_{ij}}_2^2 \\
        \dot{x}_{ij} &= f(x_{ij}, u_{ij}, w_j),\  
            x_{ij}(0) = x^*_i(0) \\
        \dot{\hat{a}}_{ij} &= \rho(x_{ij}, x^*_i, \hat{a}_{ij}; \theta),\
            \hat{a}_{ij}(0) = 0 \\
        u_{ij} &= \pi(x_{ij}, x^*_i, \hat{a}_{ij}; \theta)
    \end{aligned}
\end{aligned}.
\end{equation}
Solving~\cref{eq:meta-adaptive-control} would yield parameters~$\theta$ for the adaptive controller~$(\pi,\rho)$ such that it works well on average in closed-loop tracking of~$\{x^*_i(t)\}_{i=1}^N$ for the dynamics~$f$, subject to the disturbances~$\{w_j(t)\}_{j=1}^M$. To learn the meta-parameters~$\theta$, we can perform gradient descent on~\cref{eq:meta-adaptive-control}. This requires back-propagating through an ODE solver, which can be done either directly or via the adjoint state method after solving the ODE forward in time \citep{PontryaginBoltyanskiiEtAl1962,ChenRubanovaEtAl2018,AnderssonGillisEtAl2019,MillardHeidenEtAl2020}. In addition, the learning problem \cref{eq:meta-adaptive-control} is \emph{semi-supervised}, in that~$\{w_j(t)\}_{j=1}^M$ are labelled samples and~$\{x^*_i(t)\}_{i=1}^N$ can be chosen freely. If there are some specific target trajectories we want to track at test time, we can use them in the meta-problem~\cref{eq:meta-adaptive-control}. This is an advantage of designing the offline learning problem in the context of the downstream control objective.

\subsection{Model Ensembling as a Proxy for Feedback Offline}
\label{sec:model-ensemble}

In practice, we cannot simulate the true dynamics~$f$ or sample an actual disturbance trajectory~$w(t)$ offline. Instead, we can more reasonably assume we have past data collected with some other, possibly poorly tuned controller. In particular, we assume access to trajectory data $\{\mathcal{T}_j\}_{j=1}^M$, such that
\begin{equation}\label{eq:trajectory-data}
    \mathcal{T}_j \defn \cbr[\big]{\rbr[\big]{
        t^{(k)}_j, x^{(k)}_j, u^{(k)}_j, t^{(k+1)}_{j}, x^{(k+1)}_{j}
    }}_{k=0}^{N_j-1},
\end{equation}
where~$x^{(k)}_j \defn x_j(t^{(k)}_j)$ and $u^{(k)}_j \defn u_j(t^{(k)}_j)$ were the state and control input, respectively, at time $t^{(k)}_j$.

Inspired by \citet{ClaveraRothfussEtAl2018}, since we cannot simulate the true dynamics~$f$ offline, we propose to first train a \emph{model ensemble} from the trajectory data~$\{\mathcal{T}_j\}_{j=1}^M$ to roughly capture the distribution of $f(\cdot,\cdot,w)$ over possible values of the disturbance~$w$. Specifically, we fit a model~$\hat{f}_j(x,u; \psi_j)$ with parameters~$\psi_j$ to each trajectory~$\mathcal{T}_j$, and use this as a proxy for~$f(x,u,w_j)$ in~\cref{eq:meta-adaptive-control}. The meta-problem~\cref{eq:meta-adaptive-control} is now
\begin{equation}\label{eq:meta-adaptive-control-ensemble}
\begin{aligned}
    \min_\theta \
    &\frac{1}{NMT} \rbr*{ \sum_{i=1}^N \sum_{j=1}^M \int_0^{T} c_{ij}(t) \,dt + \mu_\mathrm{meta}\norm{\theta}_2^2 }
    \\&\mathrm{s.t.}\enspace\begin{aligned}[t]
        c_{ij} &= \norm{x_{ij} - x^*_i}_2^2 + \mu_\mathrm{ctrl}\norm{u_{ij}}_2^2 \\
        \dot{x}_{ij} &= \hat{f}_j(x_{ij}, u_{ij}; \psi_j),\  
            x_{ij}(0) = x^*_i(0) \\
        \dot{\hat{a}}_{ij} &= \rho(x_{ij}, x^*_i, \hat{a}_{ij}; \theta),\
            \hat{a}_{ij}(0) = 0 \\
        u_{ij} &= \pi(x_{ij}, x^*_i, \hat{a}_{ij}; \theta)
    \end{aligned}
\end{aligned}
\end{equation}
This form is still semi-supervised, since each model~$\hat{f}_j$ is dependent on the trajectory data~$\mathcal{T}_j$, while~$\{x^*_i\}_{i=1}^N$ can be chosen freely. The collection~$\{\hat{f}_j\}_{j=1}^M$ is termed a model ensemble. Empirically, the use of model ensembles has been shown to improve robustness to model bias and train-test data shift of deep predictive models \citep{LakshminarayananPritzelEtAl2017} and policies in reinforcement learning \citep{RajeswaranGhotraEtAl2017,KurutachClaveraEtAl2018,ClaveraRothfussEtAl2018}. To train the parameters~$\psi_j$ of model~$\hat{f}_j$ on the trajectory data~$\mathcal{T}_j$, we do gradient descent on the one-step prediction problem
\begin{equation}\label{eq:ensemble-training}
\begin{aligned}
    \min_{\psi_j} \
    &\frac{1}{N_j}\rbr*{ \sum_{k=0}^{N_j-1} \norm[\big]{x_{j}^{(k+1)} - \hat{x}_{j}^{(k)}}_2^2 + \mu_\mathrm{ensem}\norm{\psi_j}_2^2 }
    \\[-0.5em]&\mathrm{s.t.}\enspace\begin{aligned}[t]
        \hat{x}_{j}^{(k+1)} &= x_j^{(k)} + \int_{t_j^{(k)}}^{t_{j}^{(k+1)}} \hat{f}_j(x(t), u_j^{(k)}; \psi_j) \,dt
    \end{aligned}
\end{aligned}
\end{equation}
where $\mu_\mathrm{ensem} > 0$ regularizes~$\psi_j$. Since we meta-train~$\theta$ in~\cref{eq:ensemble-training} to be adaptable to every model in the ensemble, we only need to roughly characterize how the dynamics $f(\cdot,\cdot,w)$ vary with the disturbance~$w$, rather than do exact model fitting of $\hat{f}_j$ to $\mathcal{T}_j$. Thus, we approximate the integral in \cref{eq:ensemble-training} with a single step of a chosen ODE integration scheme and back-propagate through this step, rather than use a full pass of an ODE solver.

\section{Incorporating Prior Knowledge About Robot Dynamics for Principled Control-Oriented Meta-Learning}
\label{sec:control-design}

So far our method has been agnostic to the choice of adaptive controller~$(\pi,\rho)$. However, if we have some prior knowledge of the dynamical system~\cref{eq:nonlinear-system}, we can use this to make a good choice of structure for~$(\pi,\rho)$. In \cref{sec:fully-actuated,sec:control-affine}, we review adaptive control designs for two general classes of nonlinear dynamical systems. Then, in \cref{sec:applying-our-method} we discuss how we can apply our meta-learning methodology from \cref{sec:method} to such designs.

\subsection{Fully-Actuated Lagrangian Systems}\label{sec:fully-actuated}
First, we consider the class of \emph{fully-actuated Lagrangian dynamical systems}, which includes robots such as manipulator arms and multicopters. The state of such a system is $x \defn (q,\dot{q})$, where $q(t) \in \R^d$ is the vector of generalized coordinates or degrees of freedom completely describing the configuration of the system at time~$t \in \R_{\geq 0}$. The nonlinear dynamics of such systems are fully described by
\begin{equation}\label{eq:lagrange}
    H(q)\ddot{q} + C(q,\dot{q})\dot{q} + g(q) = \tau_{q,\dot{q}}(u) + f_\mathrm{ext}(q,\dot{q}),
\end{equation}
where $H(q)$ is the positive-definite inertia matrix, $C(q,\dot{q})$ is the Coriolis matrix, $g(q)$ is the potential force, $\tau_{q,\dot{q}}(u)$ is the generalized input force with invertible map ${\tau_{q,\dot{q}} : \R^d \to \R^d}$ parameterized by the state $(q,\dot{q})$, and $f_\mathrm{ext}(q,\dot{q})$ summarizes any other external generalized forces. \citet{SlotineLi1987} studied adaptive control for~\cref{eq:lagrange} under the assumptions:
\begin{itemize}
    \item The matrix ${ \dot{H}(q,\dot{q}) - 2C(q,\dot{q}) }$ is always skew-symmetric. Since the vector $C(q,\dot{q})\dot{q}$ is uniquely defined, $C(q,\dot{q})$ can always be chosen such that this assumption holds \citep{SlotineLi1991}.

    \item The dynamics in~\cref{eq:lagrange} are linearly parameterizable, \ie
    \begin{equation}
        H(q)\dot{v} + C(q,\dot{q})v + g(q) - f_\mathrm{ext}(q,\dot{q}) 
        = Y(q,\dot{q},v,\dot{v})a,
    \end{equation}
    for some known matrix~$Y(q,\dot{q},v,\dot{v}) \in \R^{d \x p}$, any vectors $q,\dot{q},v,\dot{v} \in \R^d$, and constant unknown parameters~${a \in \R^p}$.
    \item The target trajectory for $x \defn (q,\dot{q})$ is of the form ${x^* = (q^*,\dot{q}^*)}$, where~$q^*(t)$ is twice-differentiable.
\end{itemize}
Under these assumptions, the adaptive controller
\begin{equation}\label{eq:slotine-li-controller}
\begin{aligned}[c]
    \tilde{q} &\defn q - q^* 
    \\
    s &\defn \dot{\tilde{q}} + \Lambda{\tilde{q}}
    \\
    v &\defn \dot{q}^* - \Lambda{\tilde{q}}
\end{aligned}\qquad\begin{aligned}[c]
    u &= \inv{\tau}_{q,\dot{q}}\rbr*{ Y(q,\dot{q},v,\dot{v})\hat{a} - Ks }
    \\
    \dot{\hat{a}} &= -\Gamma \tran{Y(q,\dot{q},v,\dot{v})} s
\end{aligned}
\end{equation}
ensures $x(t) = (q(t),\dot{q}(t))$ converges asymptotically to ${x^*(t) = (q^*(t),\dot{q}^*(t))}$, where $\tilde{q}, s, v \in \R^d$ are auxiliary variables and $(\Lambda, K, \Gamma)$ are any chosen positive-definite gain matrices. The proof of tracking convergence from \citet{SlotineLi1987} relies on the Lyapunov candidate function
\begin{equation}\label{eq:lyapunov-lagrange}
    V(q,\dot{q},q^*,\dot{q}^*,\hat{a},a) = \frac{1}{2}\tran{s}H(q)s + \frac{1}{2}\norm{\hat{a} - a}_{\inv{\Gamma}}^2,
\end{equation}
which comprises the energy term $\frac{1}{2}\tran{s}H(q)s$ that quantifies the distance between $(q,\dot{q})$ and $(q^*,\dot{q}^*)$, and a parameter error term $\frac{1}{2}\norm{\hat{a} - a}_{\inv{\Gamma}}^2$ weighted by $\inv{\Gamma} \succ 0$.

\subsection{Control-Affine Systems with Matched Uncertainty}\label{sec:control-affine}

In the previous section, control of the Lagrangian dynamics~\cref{eq:lagrange} was facilitated by their fully-actuated nature, i.e., by our ability to directly command the acceleration~$\ddot{q}$ with the control input. As a result, we could track any twice-differentiable target trajectory~$q^*$.

We now consider the broad class of nonlinear, control-affine dynamical systems of the form 
\begin{equation}\label{eq:control-affine}
    \dot{x} = f(x) + B(x)(u + g_\mathrm{ext}(x)),
\end{equation}
with state $x \in \R^n$ and control input $u \in \R^m$ such that ${m < n}$. We assume the functions ${f : \R^n \to \R^n}$ and ${B : \R^n \to \R^{n \x m}}$ are known, and $\rank(B(x)) = m$ for all~$x$. The function $g_\mathrm{ext} : \R^n \to \R^m$ representing external influences on the system is \emph{unknown}. The quantity $g_\mathrm{ext}(x)$ is termed a \emph{matched uncertainty} since, if it was known, it could be directly cancelled by the control input $u = -g_\mathrm{ext}(x)$.

We cannot directly influence the state derivative $\dot{x}$ through the input $u$ since $B(x)$ in \cref{eq:control-affine} is not invertible. As such, we cannot hope to track arbitrary target trajectories. Rather, we aim to make $x(t)$ track some target trajectory $x^*(t)$ from a known state-input pair $(x^*,u^*)$ that is \emph{nominally open-loop feasible}, i.e., satisfies the \emph{nominal dynamics}
\begin{equation}\label{eq:control-affine-nominal}
    \dot{x} = f(x) + B(x)u.
\end{equation}
If we knew $g_\mathrm{ext}$, the pair $(x^*,u^* - g_\mathrm{ext}(x^*))$ would then be open-loop feasible for the true dynamics~\cref{eq:control-affine}.

Fully-actuated Lagrangian systems nearly fit into the form of \cref{eq:control-affine-nominal} if we write \cref{eq:lagrange} with $f_\mathrm{ext}(q,\dot{q}) \equiv 0$ as
\begin{equation}\label{eq:fully-actuated-control-affine}
\begin{aligned}
    \dot{x} = \pmx{\dot{q} \\ -\inv{H(q)}( C(q,\dot{q})\dot{q} + g(q) ) } + \bmx{0 \\ \inv{H(q)}}\tau_{q,\dot{q}}(u),
\end{aligned}\end{equation}
where $x \defn (q,\dot{q})$. Then, given a twice-differentiable trajectory $q^*(t)$, any state-input pair $(x^*,u^*)$ of the form
\begin{equation}\begin{aligned}
    x^* &= (q^*, \dot{q}^*) \\
    u^* &= \inv{\tau}_{q^*,\dot{q}^*}\rbr*{ H(q^*)\ddot{q}^* + C(q^*,\dot{q}^*)\dot{q}^* + g(q^*) }
\end{aligned}\end{equation}
is open-loop feasible. However, \cref{eq:control-affine} also includes \emph{underactuated} Lagrangian systems, e.g.,
\begin{equation}\label{eq:underactuated}
\begin{aligned}
    \dot{x} = \pmx{\dot{q} \\ -\inv{H(q)}( C(q,\dot{q})\dot{q} + g(q) ) } + \bmx{0 \\ \inv{H(q)}}F(q,\dot{q})u,
\end{aligned}\end{equation}
where $F(q,\dot{q}) \in \R^{d \x m}$ with $m < d$, such that the map $\tau_{q,\dot{q}} \equiv F(q,\dot{q})$ is \emph{not} invertible.

Unlike fully-actuated Lagrangian systems, there is no universal control design that can stabilize systems of the forms~\cref{eq:control-affine-nominal} and~\cref{eq:underactuated}. Instead, controllers for such systems must be designed on a case-by-case basis. However, we can design a general adaptation law for such systems if a feedback controller and accompanying \emph{stability certificate} have already been designed for the nominal dynamics~\cref{eq:control-affine-nominal}, and if $g_\mathrm{ext}$ in~\cref{eq:control-affine} is linearly parameterizable, i.e.,
\begin{equation}\label{eq:control-affine-parametric}
    \dot{x} = f(x) + B(x)(u + Y(x)a),
\end{equation}
where $Y : \R^n \to \R^{m \x p}$ is known, and $a \in \R^p$ is \emph{unknown} yet fixed. To this end, we now present \cref{thm:adaptive-lyapunov}.

\begin{proposition}\label{thm:adaptive-lyapunov}
    Let $(x,u)$ be a state-input pair satisfying the dynamics \cref{eq:control-affine-parametric} with fixed parameters $a \in \R^p$, and let $(x^*,u^*)$ be a state-input pair satisfying the nominal dynamics~\cref{eq:control-affine-nominal}. Consider the quantity
    \begin{equation}
        V(x,x^*,\hat{a},a) \defn \bar{V}(x,x^*) + \frac{1}{2}\norm{\hat{a} - a}_{\inv{\Gamma}}^2,
    \end{equation}
    for any $\bar{V} : \R^n \x \R^n \to \R$ and $\Gamma \succ 0$, with $\hat{a}(t) \in \R^p$. Suppose we apply the adaptive controller
    \begin{equation}\label{eq:adaptive-controller-affine}
    \begin{aligned}
        u &= \bar{u} - Y(x)\hat{a} \\
        \dot{\hat{a}} &= \Gamma \tran{Y(x)}\tran{B(x)}\grad_x \bar{V}(x,x^*)
    \end{aligned}\end{equation}
    to the dynamics \cref{eq:control-affine-parametric} for any $\bar{u}(t) \in \R^m$. Then
    \begin{equation}\begin{aligned}
        \dot{V} = &\tran{\grad_x \bar{V}(x,x^*)}(f(x) + B(x)\bar{u}) \\
                  &+ \tran{\grad_{x^*} \bar{V}(x,x^*)}(f(x^*) + B(x^*)u^*)
    \end{aligned}~.\end{equation}
\end{proposition}
\begin{proof}
    Taking the time derivative of $V(x,x^*,\hat{a})$ along any trajectory $(x(t),x^*(t),\hat{a}(t))$, we have
    \begin{equation}\begin{aligned}
        \dot{V}
        &= \tran{\grad_x\bar{V}(x,x^*)}\dot{x} + \tran{\grad_{x^*}\bar{V}(x,x^*)}\dot{x}^* + \tran{\dot{\hat{a}}}\inv{\Gamma}(\hat{a} - a)
        \\
        &= \begin{aligned}[t]
            &\tran{\grad_x\bar{V}(x,x^*)}\rbr*{f(x) + B(x)(u + Y(x)a)} \\
            &+ \tran{\grad_{x^*}\bar{V}(x,x^*)}\rbr*{f(x^*) + B(x^*)u^*} \\
            &+ \tran{\dot{\hat{a}}}\inv{\Gamma}(\hat{a} - a)
        \end{aligned}
        \\
        &= \begin{aligned}[t]
            &\tran{\grad_x\bar{V}(x,x^*)}\rbr*{f(x) + B(x)\bar{u} - Y(x)(\hat{a} - a))} \\
            &+ \tran{\grad_{x^*}\bar{V}(x,x^*)}\rbr*{f(x^*) + B(x^*)u^*} \\
            &+ \tran{\dot{\hat{a}}}\inv{\Gamma}(\hat{a} - a)
        \end{aligned}
        \\
        &= \begin{aligned}[t]
            &\tran{\grad_x\bar{V}(x,x^*)}\rbr*{f(x) + B(x)\bar{u}} \\
            &+ \tran{\grad_{x^*}\bar{V}(x,x^*)}\rbr*{f(x^*) + B(x^*)u^*} \\
            &+ \tran{\rbr*{ \dot{\hat{a}} - \Gamma\tran{Y(x)}\tran{B(x)}\grad_x\bar{V}(x,x^*) }}\inv{\Gamma}(\hat{a} - a)
        \end{aligned}
        \\
        &= \begin{aligned}[t]
            &\tran{\grad_x\bar{V}(x,x^*)}\rbr*{f(x) + B(x)\bar{u}} \\
            &+ \tran{\grad_{x^*}\bar{V}(x,x^*)}\rbr*{f(x^*) + B(x^*)u^*}
        \end{aligned}
    \end{aligned}~,\end{equation}
    as required.
\end{proof}

Essentially, the adaptive controller \cref{eq:adaptive-controller-affine} ensures the scalar quantity $V(x,x^*,\hat{a},a)$ evolves along trajectories of the dynamics \cref{eq:control-affine-parametric} in the same manner as $\bar{V}(\bar{x},x^*)$ does with $(\bar{x},\bar{u})$ satisfying the nominal dynamics~\cref{eq:control-affine-nominal}, for any nominal control signal $\bar{u}(t) \in \R^m$.

As a result, the objectives of controller and adaptation design are de-coupled. We can first design a stabilizing feedback policy $\bar{u} = \bar{\pi}(\bar{x}, x^*, u^*)$ such that any trajectory $\bar{x}(t)$ of the closed-loop nominal system ${\dot{\bar{x}} = f(\bar{x}) + B(\bar{x})\bar{\pi}(\bar{x}, x^*, u^*)}$ is guaranteed to converge to~$x^*(t)$ by the accompanying certificate function~$\bar{V}$. Then, we can augment $(\bar{\pi}, \bar{V})$ with \cref{thm:adaptive-lyapunov} to adaptively stabilize the true dynamics \cref{eq:control-affine-parametric}.

\citet{BoffiTuEtAl2020} discuss the case where $\bar{V}$ is a Lyapunov function in the sense of Lyapunov's direct method \citep{Lyapunov1892} with $x^*(t) \equiv 0$. However, we stress that \cref{thm:adaptive-lyapunov} holds for any scalar quantity $\bar{V}(x,x^*)$, and thus can be used to construct adaptive controllers from a function $\bar{V}$ that certifies stability of the nominal dynamics predicated on other ``Lyapunov-like'' analysis tools, such as LaSalle's invariance principle \citep{LaSalle1960}, Barb\u{a}lat's lemma \citep{Barbalat1959}, or incremental stability \citep{LohmillerSlotine1998,Angeli2002}.

\subsection{Control-Oriented Meta-Learning for Nonlinear Adaptive Control}\label{sec:applying-our-method}

In \cref{sec:method}, we presented our framework for control-oriented meta-learning abstracted to learn a general parametric feedback controller $u = \pi(x,x^*,\hat{a};\theta)$ and adaptation law $\dot{\hat{a}} = \rho(x,x^*,\hat{a};\theta)$ with data collected from a dynamical system $\dot{x} = f(x,u,w)$. Then, in \cref{sec:fully-actuated,sec:control-affine} we studied particular adaptive controller designs for two large classes of nonlinear dynamical systems. Now, we instantiate our control-oriented meta-learning method alongside these principled adaptive controller designs.

\subsubsection{Fully-Actuated Lagrangian Systems}
The adaptive controller~\cref{eq:slotine-li-controller} requires the nonlinearities in the dynamics~\cref{eq:lagrange} to be known a priori. While \citet{NiemeyerSlotine1991} showed these can be systematically derived for~$H(q)$, $C(q,\dot{q})$, and~$g(q)$, there exist many external forces~$f_\mathrm{ext}(q,\dot{q})$ of practical importance in robotics for which this is difficult to do, such as aerodynamic and contact forces. Thus, we consider the case when~$H(q)$, $C(q,\dot{q})$, and~$g(q)$ are known and~$f_\mathrm{ext}(q,\dot{q})$ is unknown. Moreover, we want to approximate~$f_\mathrm{ext}(q,\dot{q})$ with the neural network
\begin{equation}\label{eq:nn_lagrange}
    \hat{f}_\mathrm{ext}(q,\dot{q}; \hat{A}, \varphi) \defn \hat{A}y(q,\dot{q}; \varphi),
\end{equation}
where the features $y(q,\dot{q}; \varphi) \in \R^p$ consist of all the hidden layers of the network parameterized by~$\varphi$, and~$\hat{A} \in \R^{d \x p}$ is the output layer. Inspired by~\cref{eq:slotine-li-controller}, we consider the adaptive controller
\begin{equation}\label{eq:slotine-li-controller-parametric}
\begin{aligned}[c]
    \tilde{q} &\defn q - q^* 
    \\
    s &\defn \dot{\tilde{q}} + \Lambda{\tilde{q}}
    \\
    v &\defn \dot{q}^* - \Lambda{\tilde{q}}
\end{aligned}\qquad\begin{aligned}[c]
    \bar{\tau} &\defn H(q)\dot{v} + C(q,\dot{q})v + g(q) - Ks
    \\
    u &= \inv{\tau}_{q,\dot{q}}\rbr*{ \bar{\tau} - \hat{A}y(q,\dot{q}; \varphi) }
    \\
    \dot{\hat{A}} &= \Gamma s \tran{y(q,\dot{q}; \varphi)}
\end{aligned}.
\end{equation}
If $f_\mathrm{ext}(q,\dot{q}) = \hat{f}_\mathrm{ext}(q,\dot{q}; \hat{A}, \varphi)$ for fixed values~$\varphi$ and~$\hat{A}$, then the adaptive controller~\cref{eq:slotine-li-controller-parametric} guarantees tracking convergence \citep{SlotineLi1987}. In general, we do not know such a value for~$\varphi$, and we must choose the gains~$(\Lambda, K, \Gamma)$. Since~\cref{eq:slotine-li-controller-parametric} is parameterized by $\theta \defn (\varphi, \Lambda, K, \Gamma)$, we can meta-learn~\cref{eq:slotine-li-controller-parametric} with the method described in \cref{sec:method}. In order to include the positive-definite gains $(\Lambda, K, \Gamma)$ in any gradient-descent-based training loop, we can use an unconstrained parameterization for each gain matrix. Specifically, any $n \times n$ positive-definite matrix~$Q$ can be uniquely defined by~${\frac{1}{2}n(n+1)}$ unconstrained parameters. We use the log-Cholesky parameterization \citep{PinheiroBates1996} for any positive-definite matrix $Q$, which takes the form $Q = L\tran{L}$ with the lower-triangular matrix
\begin{equation}\setlength{\arraycolsep}{2pt}\hspace{-0.5em}
    L = \bmx{
        \exp\theta_1                    &0                          &\cdots                     &\cdots                     &0 \\
        \theta_{n+1}                    &\exp\theta_2               &0                          &\rotatebox{10}{$\ddots$}   &\vdots \\
        \theta_{n+2}                    &\theta_{n+3}               &\exp\theta_3               &\rotatebox{10}{$\ddots$}   &\vdots \\
        \vdots                          &\rotatebox{10}{$\ddots$}   &\rotatebox{10}{$\ddots$}   &\rotatebox{10}{$\ddots$}   &0 \\
        \theta_{\frac{1}{2}n(n-1) + 2}  &\cdots                     &\cdots                     &\theta_{\frac{1}{2}n(n+1)} &\exp\theta_n
    }
\end{equation}
formed using unconstrained parameters $\theta \in \R^{\frac{1}{2}n(n+1)}$.

While for simplicity we consider known $H(q)$, $C(q,\dot{q})$, and $g(q)$, we can extend to the case when they are linearly parameterizable, \eg $H(q)\dot{v} + C(q,\dot{q})v + g(q) = Y(q,\dot{q},v,\dot{v})a\ $ with $Y(q,\dot{q},v,\dot{v})$ a fixed feature matrix consisting of, e.g., those systematically computed with prior knowledge of the dynamics \citep{NiemeyerSlotine1991}. In this case, we would then maintain a separate adaptation law~${\dot{\hat{a}} = -P \tran{Y(q,\dot{q},v,\dot{v})} s}$ with adaptation gain $P \succ 0$ in our proposed adaptive controller~\cref{eq:slotine-li-controller-parametric}. It is also possible to construct additional features as products of known terms with parametric \citep{SannerSlotine1995} and non-parametric \citep{BoffiTuEtAl2021} approximators.

\subsubsection{Underactuated Control-Affine Systems} 
In a manner similar to~\cref{eq:nn_lagrange}, we can approximate~$g_\mathrm{ext}(x)$ in~\cref{eq:control-affine} with the neural network
\begin{equation}\label{eq:nn_underactuated}
    \hat{g}_\mathrm{ext}(x; \hat{A}, \varphi) \defn \hat{A}y(x; \varphi),
\end{equation}
with features $y(x; \varphi)\in \R^p$ and output layer~${\hat{A} \in \R^{m \x p}}$. Then, given a stabilizing feedback controller ${\bar{\pi} : \R^n \x \R^n \x \R^m \to \R^m}$ and accompanying certificate function ${\bar{V} : \R^n \x \R^n \to \R}$ for the nominal dynamics \cref{eq:control-affine-nominal}, we can apply the proof of \cref{thm:adaptive-lyapunov} to the augmented certificate function
\begin{equation}
    V(x,x^*,\hat{A},A) \defn \bar{V}(x,x^*) + \frac{1}{2}\norm{\hat{A} - A}_{\inv{\Gamma}}^2,
\end{equation}
with the squared weighted trace norm
\begin{equation}
    \norm{\hat{A} - A}_{\inv{\Gamma}}^2 \defn \tr\rbr*{ \tran{(\hat{A} - A)} \inv{\Gamma} (\hat{A} - A) }.
\end{equation}
The result is the adaptive controller
\begin{equation}\label{eq:adaptive-control-affine-parametric}
\begin{aligned}
    u &= \bar{\pi}(x,x^*,u^*; \kappa) - \hat{A}y(x; \varphi) \\
    \dot{\hat{A}} &= \Gamma \tran{B(x)}\grad_x \bar{V}(x,x^*; \kappa)\tran{y(x; \varphi)}
\end{aligned}~,\end{equation}
where we have included the placeholder parameters~$\kappa$ for any control gains analogous to $(\Lambda, K)$ in~\cref{eq:slotine-li-controller-parametric}. The closed-loop system induced by \cref{eq:adaptive-control-affine-parametric} is parameterized overall by $\theta \defn (\varphi, \kappa)$, which can thus be meta-learned with the method described in~\cref{sec:method}.

\section{Experiments}\label{sec:experiments}
We evaluate our method in simulation on two example systems: a Planar Fully-Actuated Rotorcraft (PFAR), and the classic \emph{underactuated} Planar Vertical Take-Off and Landing (PVTOL) vehicle from \citet{HauserSastryEtAl1992}. The PFAR system dynamics are governed by the nonlinear equations of motion
\begin{equation}\label{eq:pfar}
    \pmx{\ddot{x} \\ \ddot{y} \\ \ddot{\phi}}
    = \underbrace{\pmx{0 \\ -g \\ 0}}_{\eqqcolon f_g} + \underbrace{\bmx{
        \cos\phi  & -\sin\phi   & 0 \\
        \sin\phi  &  \cos\phi   & 0 \\
        0         & 0           & 1
    }}_{\eqqcolon R(\phi)}u + f_\mathrm{ext},
\end{equation}
while the PVTOL system dynamics are given by
\begin{equation}\label{eq:pvtol}
    \pmx{\ddot{x} \\ \ddot{y} \\ \ddot{\phi}}
    = f_g + \underbrace{R(\phi)\bmx{0 & 0 \\ 1 & 0 \\ 0 & 1}}_{\eqqcolon B(\phi)}u + f_\mathrm{ext}.
\end{equation}
Both systems have the degrees of freedom ${q \defn (x,y,\phi)}$, where $(x,y)$ is the position of the center of mass in the inertial frame and $\phi$ is the roll angle. In addition, ${ g = 9.81~\mathrm{m/s^2} }$ is the gravitational acceleration, $f_g \in \R^3$ is the mass-normalized gravitational force in vector form, $R(\phi)$ is a rotation matrix between inertial and body-fixed frames, and $f_\mathrm{ext}$ is some unknown external mass-normalized force.

The PFAR system is fully-actuated with control input ${ u \in \R^3 }$ that directly commands the thrust along the body-fixed $x$-axis, thrust along the body-fixed $y$-axis, and torque about the center of mass. We depict an exemplary PFAR design in \cref{fig:pfar} inspired by thriving interest in fully- and over-actuated multirotor vehicles in the robotics literature \citep{RyllMuscioEtAl2017,KamelVerlingEtAl2018,BrescianiniDAndrea2018,ZhengTanEtAl2020,RashadGoerresEtAl2020}. On the other hand, the PVTOL system is \emph{underactuated} with control input $u \in \R^2$ that directly commands only the thrust along the body $x$-axis and torque about the center of mass.

In our simulations for both systems, $f_\mathrm{ext}$ is a mass-normalized quadratic drag force, due to the velocity of the vehicle relative to wind blowing at a velocity~$w \in \R$ along the inertial $x$-axis. Specifically, this drag force is described by the equations
\begin{equation}\label{eq:drag}
\begin{aligned}
    v_x &\defn (\dot{x} - w)\cos\phi + \dot{y}\sin\phi 
    \\
    v_y &\defn -(\dot{x} - w)\sin\phi + \dot{y}\cos\phi 
    \\
    v_L &\defn -\dot{\phi} - v_y
    \\
    v_R &\defn \dot{\phi} - v_y
    \\
    f_\mathrm{ext} &= - \tran{ R(\phi) }\pmx{\beta_1v_1\abs{v_1} \\ \beta_2v_2\abs{v_2} \\ \beta_3(v_R\abs{v_R} - v_L\abs{v_L})}
\end{aligned},\end{equation}
where $\beta_1, \beta_2, \beta_3 > 0$ are aggregate drag coefficients. Each body-fixed component of the drag force opposes a corresponding body-fixed component of either the linear or rotational velocity of the vehicle. In the case of the PVTOL system, $f_\mathrm{ext}$ is a matched uncertainty if and only if $\beta_1 = 0$; instead, we use a small non-zero value of $\beta_1$ to make $f_\mathrm{ext}$ an \emph{unmatched} uncertainty and therefore a difficult disturbance for the PVTOL system to overcome. In particular, we use ${\beta_1 = 0.01}$, ${\beta_2 = 1}$, and ${\beta_3 = 1}$ in all of our simulations.

\subsection{Nominal Feedback Control}
To meta-learn an adaptive controller using our method described in \cref{sec:method}, we require trajectory data $\{\mathcal{T}_j\}_{j=1}^M$ of the form \cref{eq:trajectory-data} collected on the system of interest. In order to collect this data for either the PFAR system or PVTOL system, we need to: 1) derive a feedback controller using the known nominal dynamics, 2) generate nominally open-loop feasible trajectories, and 3) record state and input measurements in closed-loop with the feedback controller while trying to track the generated trajectories. That is, in this stage, we do \emph{not} know anything about $f_\mathrm{ext}$; in a simulation environment, this amounts to the details of $f_\mathrm{ext}$ in~\cref{eq:drag} being unavailable to any feedback controller. Thus, we must do our best to collect data using a feedback controller derived for the nominal dynamics (i.e., with ${f_\mathrm{ext} \equiv 0}$). 

\subsubsection{PFAR Feedback Control}
To this end, for the PFAR system we rely on the Proportional-Derivative (PD) controller with feed-forward described by
\begin{equation}\label{eq:pd}
    u = \tran{R(\phi)}\rbr*{
        \ddot{q}^* - f_g - K_P\tilde{q} - K_D\dot{\tilde{q}}
    },
\end{equation}
with gains $K_P, K_D \succ 0$. For $f_\mathrm{ext}(q,\dot{q}) \equiv 0$, this controller feedback-linearizes the dynamics~\cref{eq:pfar} around any twice-differentiable trajectory~$q^*(t)$, such that the error ${\tilde{q} \defn q - q^*}$ is governed by an exponentially stable ODE.

\subsubsection{PVTOL Feedback Control} Feedback control for the PVTOL system is complicated considerably by the underactuated nature of its dynamics. To begin, we can generate nominally open-loop feasible trajectories by leveraging the fact that the PVTOL dynamics \cref{eq:pvtol} with $f_\mathrm{ext} \equiv 0$ are differentially flat with flat outputs $(x,y)$ \citep{Ailon2010}. Indeed, given any four-times-differentiable position trajectory ${(x^*(t), y^*(t)) \in \R^2}$, the state-input pair $\rbr{ (x^*, y^*, \phi^*, \dot{x}^*, \dot{y}^*, \dot{\phi}^*), (u_1^*, u_2^*) }$ satisfying
\begin{equation}\begin{aligned}
    \phi^* &= \arctan\rbr*{ \frac{-\ddot{x}^*}{\ddot{y}^*{+}g} }
    \\[0.5em]
    u_1^* &= \sqrt{(\ddot{x}^*)^2 + (\ddot{y}^*{+}g)^2}
    \\[0.5em]
    \dot{\phi}^* &= \frac{ \ddot{x}^*\dddot{y}^* - \dddot{x}^*(\ddot{y}^*{+}g) }{ (u_1^*)^2 }
    \\[0.5em]
    u_2^* &= \frac{ \ddot{x}^*\ddddot{y}^* - \ddddot{x}^*(\ddot{y}^*{+}g) - 2\dot{\phi}^*(\ddot{x}^*\dddot{x}^* + (\ddot{y}^*{+}g)\dddot{y}^*) }{ (u_1^*)^2 }
\end{aligned}\end{equation}
is open-loop feasible for the nominal PVTOL dynamics in~\cref{eq:pvtol}. Given such a trajectory $(x^*(t), y^*(t))$, \citet{Ailon2010} proved that the feedback controller described by
\begin{equation}\label{eq:pvtol-feedback}
\begin{aligned}
    \tilde{x} &\defn x - x^*
    \\
    \tilde{y} &\defn y - y^*
    \\
    v_x &\defn \ddot{x}^* - c_{x1}\tanh(k_{x1}\tilde{x}{+}k_{x2}\dot{\tilde{x}}) 
        - c_{x2}\tanh({k_{x2}\dot{\tilde{x}}})
    \\
    v_y &\defn \ddot{y}^* - c_{y1}\tanh(k_{y1}\tilde{y}{+}k_{y2}\dot{\tilde{y}}) 
        - c_{y2}\tanh({k_{y2}\dot{\tilde{y}}})
    \\
    v_\phi &\defn \arctan\rbr*{ \frac{-v_x}{v_y + g} }
    \\
    u_1 &= \sqrt{v_x^2 + (v_y + g)^2}
    \\
    u_2 &= \ddot{v}_\phi - k_{\phi 1}(\phi - v_\phi) - k_{\phi 2}(\dot{\phi} - \dot{v}_\phi)
\end{aligned}~,\end{equation}
with gains $c_x, c_y, k_x, k_y, k_\phi \in \R^2_{>0}$, ensures that the state $(x, y, \phi, \dot{x}, \dot{y}, \dot{\phi})$ asymptotically converges to the target trajectory $(x^*, y^*, \phi^*, \dot{x}^*, \dot{y}^*, \dot{\phi}^*)$. To do this, the feedback controller~\cref{eq:pvtol-feedback} specifies bounded virtual inputs $(v_x, v_y)$ with a thrust $u_1$ and desired roll angle $v_\phi$ that together would ensure convergence of the $(x,y)$-subsystem to the desired trajectory. Then, since we can only command $\ddot{\phi}$ and not  $\phi$ directly, the feedback controller~\cref{eq:pvtol-feedback} applies an outer PD loop via $u_2$ to regulate $\phi$ towards $v_\phi$.

Later in \cref{sec:baselines}, we will need a stability certificate function~$\bar{V}$ to accompany the nominal feedback controller~\cref{eq:pvtol-feedback} in comprising a meta-learned adaptive controller of the form~\cref{eq:adaptive-control-affine-parametric}. To this end, we now take a moment to identify such a certificate. The proof from \citet{Ailon2010} of asymptotic tracking convergence when using~\cref{eq:pvtol-feedback} relies on the component certificate functions 
\begin{equation}\label{eq:pvtol-lyapunov-xy}
\begin{aligned}
    \bar{V}_x &\defn\begin{aligned}[t]
        &c_{x1}\log\cosh(k_{x1}x + k_{x2}\dot{x}) \\
        &+ c_{x2}\log\cosh({k_{x2}\dot{x}}) + \frac{k_{x1}}{2}\dot{x}^2
    \end{aligned}
    \\
    \bar{V}_y &\defn\begin{aligned}[t]
        &c_{y1}\log\cosh(k_{y1}y + k_{y2}\dot{y}) \\
        &+ c_{y2}\log\cosh({k_{y2}\dot{y}}) + \frac{k_{y1}}{2}\dot{y}^2
    \end{aligned}
\end{aligned}~.\end{equation}
The stability of the linear second-order ODE for the roll dynamics induced by the choice of $u_2$ in \cref{eq:pvtol-feedback} is verified by the quadratic component Lyapunov function
\begin{equation}\label{eq:pvtol-lyapunov-phi}
\begin{aligned}
    \bar{V}_\phi &\defn \frac{1}{2}\tran{ \pmx{\tilde{\phi} \\ \dot{\tilde{\phi}}} }
    \bmx{k_{\phi 1}(k_{\phi 1} + 1) + k_{\phi 2}^2 & k_{\phi 2} \\ 
         k_{\phi 2} & k_{\phi 1} + 1}
    \pmx{\tilde{\phi} \\ \dot{\tilde{\phi}}}
\end{aligned}\end{equation}
with $\tilde{\phi} \defn \phi - v_\phi$ \citep[Example~5.12]{AstromMurray2020}. In a manner akin to backstepping \citep[Lemma~14.2]{Khalil2002}, we can add the component certificate functions from \cref{eq:pvtol-lyapunov-xy} and \cref{eq:pvtol-lyapunov-phi} to get the overall certificate function
\begin{equation}\label{eq:pvtol-lyapunov}
    \bar{V} \defn \bar{V}_x + \bar{V}_y + \bar{V}_\phi
\end{equation}
for the nominal PVTOL dynamics in closed loop with the feedback controller~\cref{eq:pvtol-feedback}.

\citet{Ailon2010} shows that $\bar{V}_x$ and $\bar{V}_y$ are globally positive-definite for the $(x,y)$-subsystem, yet $\dot{\bar{V}}_x$ and $\dot{\bar{V}}_y$ are \emph{not} globally negative-definite as a consequence of coupling with the roll dynamics of the PVTOL systems. As a result, Lyapunov's direct method cannot be used in a straightforward manner to show tracking convergence; instead, \citet{Ailon2010} shows that $\dot{\bar{V}}_x$ and $\dot{\bar{V}}_y$ always become negative in finite-time and remain negative thereafter. Regardless, as discussed in \cref{sec:control-affine} and \cref{thm:adaptive-lyapunov}, we can still use the certificate function~\cref{eq:pvtol-lyapunov} to construct an adaptive controller for the PVTOL system.

\subsection{Data Collection and Meta-Training}\label{sec:data-collection-training}
In the previous section, we derived feedback controllers for tracking nominally open-loop feasible trajectories from prior knowledge of each dynamical system for $f_\mathrm{ext} \equiv 0$. In this section, we describe how we generate such trajectories and use these feedback controllers in closed loop with the true dynamics (i.e., \cref{eq:pfar} and \cref{eq:pvtol}) to collect trajectory data~$\{\mathcal{T}_j\}_{j=1}^M$ of the form~\cref{eq:trajectory-data}.

Generating nominally feasible open-loop trajectories for the PFAR system~\cref{eq:pfar} is equivalent to generating twice-differentiable target trajectories in $(x,y,\phi)$-space. For the PVTOL system~\cref{eq:pvtol}, it is instead equivalent to generating four-times-differentiable target trajectories in $(x,y)$-space. In our experiments for either case, we use a routine to construct sufficiently smooth polynomial spline trajectories. In particular, we follow \citet{MellingerKumar2011} and \citet{RichterBryEtAl2013} in posing trajectory generation as a polynomial spline optimization problem. Specifically, we synthesize a trajectory $\mathcal{T}_j$ of training data as follows:
\begin{enumerate}
    \item Generate a uniform random walk of points in either $(x,y,\phi)$-space for the PFAR system or $(x,y)$-space for the PVTOL system.
    
    \item Fit a 30-second smooth polynomial spline trajectory~$(x^*(t),y^*(t),\phi^*(t)) \in \R^3$ for the PFAR system or~$(x^*(t),y^*(t)) \in \R^2$ for the PVTOL system to the random walk, with minimum snap in $(x^*,y^*)$ and minimum acceleration in $\phi^*$, according to the work by \citet{MellingerKumar2011} and \citet{RichterBryEtAl2013}.
    
    \item Sample a wind velocity $w$ from the training distribution in~\cref{fig:wind_dist}.
    
    \item Simulate the closed-loop dynamics of the system with the external drag force~\cref{eq:drag} to track the generated trajectory. The feedback controller has no knowledge of this drag force. For the PFAR system, we use the PD controller~\cref{eq:pd} with $K_P = 10I$ and $K_D = 0.1I$. For the PVTOL system, we use the differential-flatness-based controller~\cref{eq:pvtol-feedback} with ${c_x = c_y = k_x = k_y = k_\phi = (1,1)}$. Each of these controllers represents a ``first try'' at controlling the system in order to collect data. We record time, state, and control input measurements at~$100~\mathrm{Hz}$. 
\end{enumerate}
We record $500$~such trajectories and then sample~$M$ of them to form the training data~$\{\mathcal{T}_j\}_{j=1}^M$ for various~$M$ to evaluate the sample efficiency of our method and the baseline methods described in~\cref{sec:baselines}. That is, each trajectory $\mathcal{T}_j$ corresponds to a single wind velocity~$w$, and so a larger value of $M$ corresponds to more training data with an implicitly better representation of the training distribution in~\cref{fig:wind_dist}.

Now that we have trajectory data~$\{\mathcal{T}_j\}_{j=1}^M$ of the form~\cref{eq:trajectory-data}, we can apply our meta-learning method from \cref{sec:method} to train an adaptive controller of the form~\cref{eq:slotine-li-controller-parametric} for the PFAR system and an adaptive controller of the form~\cref{eq:adaptive-control-affine-parametric} for the PVTOL system. Specifically, the adaptive controller for the PVTOL system leverages the nominal differential-flatness-based feedback controller~\cref{eq:pvtol-feedback} and the accompanying certificate function~\cref{eq:pvtol-lyapunov}. The meta-parameters for the PFAR adaptive controller when using our meta-learning method altogether are
\begin{equation}
    \theta_\mathrm{ours} = (\varphi, \Lambda, K, \Gamma),
\end{equation}
where $\varphi$ are the parameters of the neural network features $y(q,\dot{q};\varphi)$ used in both the feedback and adaptation laws, $\Lambda, K \succ 0$ are controller gains, and $\Gamma \succ 0$ is the adaptation gain. For the PVTOL system, the meta-parameters are
\begin{equation}
    \theta_\mathrm{ours} = (\varphi, c_x, c_y, k_x, k_y, k_\phi, \Gamma),
\end{equation}
where $c_x, c_y, k_x, k_y, k_\phi \in \R^2_{> 0}$ are the controller gains and $\Gamma \succ 0$ is the adaptation gain. In both cases, our meta-learning method trains a dynamics model, controller, and adaptation law together in an end-to-end fashion. 

As we detailed in~\cref{sec:method}, offline simulation of the resulting adaptive closed-loop system yields the meta-loss function~\cref{eq:meta-adaptive-control-ensemble} of the meta-parameters~$\theta_\mathrm{ours}$. To compute the integral in~\cref{eq:meta-adaptive-control-ensemble}, we use a fourth-order Runge-Kutta scheme with a fixed time step of $0.01~\mathrm{s}$. We back-propagate gradients through this computation in a manner similar to \citet{ZhuangDvornekEtAl2020}, rather than using the adjoint method for neural ODEs \citep{ChenRubanovaEtAl2018}, due to our observation that the backward pass is sensitive to any numerical error accumulated along the forward pass during closed-loop control simulations. We present additional hyperparameter choices and training details in~\cref{app:training}.

\subsection{Baselines}\label{sec:baselines}
We compare our meta-trained adaptive controllers in trajectory tracking tasks against two types of baseline controllers.

\subsubsection{Nominal Feedback Control} Our first baseline for each system is based on the nominal controller originally used to collect data. For the PVTOL system, we simply use the non-adaptive, differential-flatness-based controller~\cref{eq:pvtol-feedback}. For the PFAR system, we use a Proportional-Integral-Derivative (PID) controller with feed-forward, \ie
\begin{equation}\label{eq:pid}
    u = \tran{R(\phi)}\rbr*{
        \ddot{q}^* - f_g - K_P\tilde{q} - K_I\!\int_0^t\tilde{q}(\eta)\,d\eta - K_D\dot{\tilde{q}}
    },
\end{equation}
with gains $K_P, K_I, K_D \succ 0$. This augments the original controller~\cref{eq:pd} with an integral term that tries to compensate for $f_\mathrm{ext}(q,\dot{q}) \not\equiv 0$. If we set ${K_P = K\Lambda + \Gamma}$, ${K_I = \Gamma\Lambda}$, and ${K_D = K + \Lambda}$, this makes the PID controller~\cref{eq:pid} equivalent to the adaptive controller~\cref{eq:slotine-li-controller-parametric} for the PFAR dynamics~\cref{eq:pfar} with $y(q,\dot{q};\varphi) \equiv 1$ (\ie constant features), $\tilde{q}(0) = 0$ (\ie zero initial position error), and ${\hat{A}(0) = 0}$ (\ie adapted parameters are initially zero). To show this, we combine the expressions in the adaptive controller~\cref{eq:slotine-li-controller-parametric} to get
\begin{equation}\begin{aligned}
    \tau_{q,\dot{q}}(u) =\ 
    &H(q)\ddot{q}^* + C(q,\dot{q})\dot{q}^* + g(q) - A(0)y(q,\dot{q}) \\
    &- (K + C(q,\dot{q}))\Lambda\tilde{q} - (H(q)\Lambda + K)\dot{\tilde{q}} \\
    &- \Gamma\rbr*{ \int_0^t \dot{\tilde{q}}(\eta)\tran{y(q(\eta),\dot{q}(\eta))}\,d\eta }y(q,\dot{q}) \\
    &- \Gamma\Lambda\rbr*{ \int_0^t \tilde{q}(\eta)\tran{y(q(\eta),\dot{q}(\eta))}\,d\eta }y(q,\dot{q})
\end{aligned}~,
\end{equation}
then set $H(q) \equiv I$, $C(q,\dot{q}) \equiv 0$, $g(q) \equiv f_g$, $\tau_{q,\dot{q}} \equiv R(\phi)$, ${ y(q,\dot{q};\varphi) \equiv 1 }$, $\tilde{q}(0) = 0$, and ${\hat{A}(0) = 0}$ to get the result
\begin{equation}\hspace{-0.5em}
\begin{aligned}
    u = \tran{R(\phi)}\big(\, \ddot{q}^* - f_g - (K\Lambda + \Gamma)\tilde{q} &- \Gamma\Lambda\!\int_0^t\tilde{q}(\eta)\,d\eta \\
                                                                             &- (K + \Lambda)\dot{\tilde{q}} \,\big)
\end{aligned}~.\end{equation}
We use this observation later in \cref{sec:results-discussion} to compare controllers with the same gains and different model features. To this end, we always set initial conditions in simulation such that~$\tilde{q}(0) = \dot{\tilde{q}}(0) = 0$ (i.e., we start on the target trajectory) and $\hat{A}(0) = 0$.

\subsubsection{Meta-Ridge Regression (MRR)} Our next baseline comes from the meta-learning work reviewed in \cref{sec:literature-meta-learning} by  \citet{HarrisonSharmaEtAl2018,HarrisonSharmaEtAl2018b}, \citet{BertinettoHenriquesEtAl2019}, \citet{LeeMajiEtAl2019}, and \citet{OConnellShiEtAl2021}, wherein ridge regression is used as a base-learner to meta-learn parametric features~$y(x;\varphi)$. That is, for a given trajectory~$\mathcal{T}_j$, these works assume the last layer~$\hat{A}$ should be the best fit in a regression sense, as a function of the parametric features~$y(x;\varphi)$, to some subset of points in~$\mathcal{T}_j$. The feature parameters $\varphi$ are then trained to minimize this regression fit. This approach, which we term \emph{Meta-Ridge Regression (MRR)}, contrasts with our thesis that~$\varphi$ should be trained for the endmost purpose of improving \emph{control} performance, rather than \emph{regression} performance.

We now specify how to implement MRR using the meta-learning language from~\cref{sec:meta-learning}. Our implementation is a generalization\footnote{
    Unlike \citet{OConnellShiEtAl2021}, we do not assume access to direct measurements of the external force~$f_\mathrm{ext}$. Also, they use a more complex form of~\cref{eq:slotine-li-controller} with better \emph{parameter estimation} properties when the dynamics are linearly parameterizable with \emph{known} nonlinear features \citep{SlotineLi1989}. While we could use a parametric form of this controller in place of~\cref{eq:slotine-li-controller-parametric}, we forgo this in favour of a simpler presentation, since we focus on offline \emph{control-oriented} meta-learning of \emph{approximate} features.
} of the approach taken by \citet{OConnellShiEtAl2021} to any nonlinear control-affine dynamical system~\cref{eq:control-affine}, which can be slightly extended using~\cref{eq:fully-actuated-control-affine} to include all fully-actuated Lagrangian systems. Given a trajectory of data~$\mathcal{T}_j$ of the form~\cref{eq:trajectory-data}, let~${\mathcal{K}_j^\mathrm{ridge} \subset \{k\}_{k=0}^{\abs{\mathcal{T}_j}-1}}$ denote the indices of transition tuples in some subset of $\mathcal{T}_j$. Define $\Delta t_j^{(k)} \defn t_{j}^{(k+1)} - t_j^{(k)}$, and the Euler approximation
\begin{equation}\label{eq:euler}
\begin{aligned}
    \hat{x}_{j}^{(k+1)}(\hat{A}) 
    &\defn x_j^{(k)} \begin{aligned}[t]
        &+ \Delta{t}_j^{(k)}\rbr*{ f(x_j^{(k)}) + B(x_j^{(k)})u_j^{(k)} } \\
        &+ \Delta{t}_j^{(k)}B(x_j^{(k)})\hat{A}y(x_j^{(k)}; \varphi)
    \end{aligned}
\end{aligned}\end{equation}
as a function of $\hat{A}$. MRR posits that $\hat{A}$ should fit some subset of the trajectory in a regression-sense; we can express this with the adaptation mechanism 
\begin{equation}\label{eq:lstsq-adapt}
    \hat{A}_j = \argmin_{\hat{A} \in \R^{m \x p}} \sum_{k \in \mathcal{K}^\mathrm{ridge}_j}
    \norm{ \hat{x}_{j}^{(k+1)}(\hat{A}) - x_j^{(k+1)} }_2^2 + \mu_\mathrm{ridge}\norm{A}_F^2.
\end{equation}
The task loss associated with trajectory $\mathcal{T}_j$ is then the regression loss
\begin{equation}\label{eq:lstsq-task-loss}
    \ell_j(\hat{A}_j, \mathcal{T}_j)
    = \frac{1}{\abs{\mathcal{T}_j}}\sum_{k = 0}^{\abs{\mathcal{T}_j} - 1} \norm{
        \hat{x}_{j}^{(k+1)}(\hat{A}) - x_j^{(k+1)}
    }_2^2
\end{equation}
The adaptation mechanism~\cref{eq:lstsq-adapt} can be solved and differentiated in closed-form for any $\mu_\mathrm{ridge} > 0$ via the normal equations, since~$\hat{x}_{j}^{(k+1)}(\hat{A})$ is linear in~$\hat{A}$; indeed, only \emph{linear} integration schemes can be substituted into~\cref{eq:euler}. The meta-problem for MRR takes the form of~\cref{eq:meta-learning} with features~$y(x;\varphi)$, the task loss~\cref{eq:lstsq-task-loss}, and the adaptation mechanism~\cref{eq:lstsq-adapt}. The meta-parameters~$\theta_\mathrm{MRR} \defn \varphi$ are trained via gradient descent on this meta-problem, and then deployed online via the features~$y(q,\dot{q};\varphi)$ in the adaptive controller~\cref{eq:slotine-li-controller-parametric} or $y(x;\varphi)$ in~\cref{eq:adaptive-control-affine-parametric}. MRR does \emph{not} meta-learn the control gains and adaptation gain, so these must still be specified by the user.

Conceptually, MRR suffers from a \emph{fundamental mismatch} between its regression-oriented meta-problem and the online problem of adaptive trajectory tracking control. The ridge regression base-learner~\cref{eq:lstsq-adapt} suggests that~$\hat{A}$ should best fit the input-output trajectory data in a regression sense. However, as we mentioned in~\cref{sec:problem-statement}, adaptive controllers such as~\cref{eq:slotine-li-controller-parametric} and~\cref{eq:adaptive-control-affine-parametric} learn on a ``need-to-know'' basis for the primary purpose of control rather than regression. As we empirically demonstrate and discuss in \cref{sec:results-discussion}, since our control-oriented approach uses a meta-objective indicative of the downstream closed-loop tracking control objective, we achieve better tracking performance than MRR at test time.

\subsection{Testing with Distribution Shift}\label{sec:testing}
\begin{figure}
    \centering
    \includegraphics[width=\columnwidth]{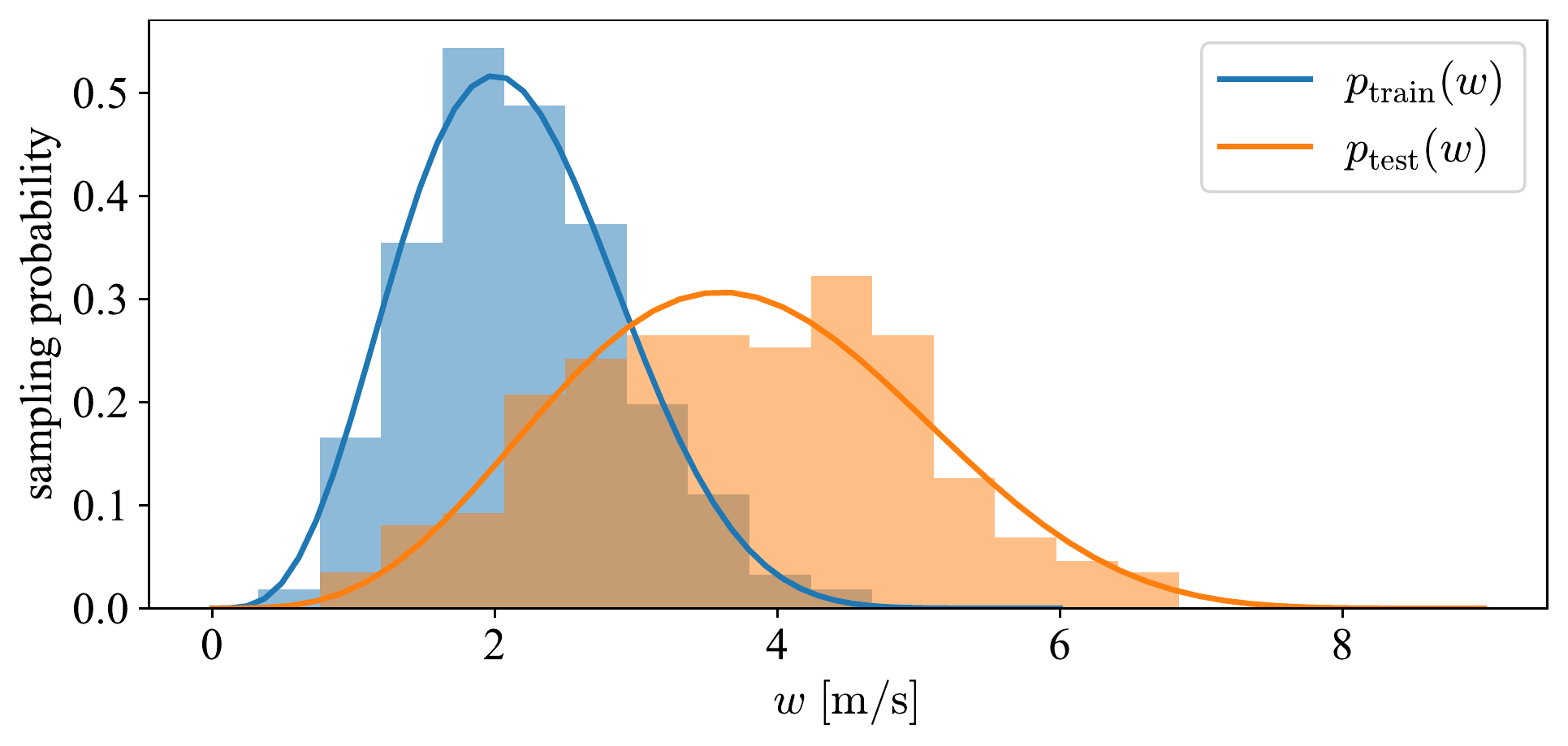}
    \caption{Training distribution~$p_\mathrm{train}$ and test distribution~$p_\mathrm{test}$ for the wind velocity~$w$ along the inertial $x$-axis. Both are scaled beta distributions; $p_\mathrm{train}$ is supported on the interval $[0, 6]~\mathrm{m/s}$ with shape parameters $(\alpha,\beta) = (5,9)$, while $p_\mathrm{test}$ is supported on the interval $[0, 9]~\mathrm{m/s}$ with shape parameters $(\alpha,\beta) = (5,7)$. The normalized histograms show the distribution of the actual wind velocity samples in the training and test data for a single seed and those in the test data, and highlight the out-of-distribution samples (\ie relative to $p_\mathrm{train}$) that occur during testing.}
    \label{fig:wind_dist}
\end{figure}
In \cref{sec:results-discussion}, we will present test results for closed-loop tracking control simulations. In particular, we want to assess the ability of each adaptive controller to \emph{generalize} to conditions different from those experienced during training data collection. To this end, during testing we always sample wind velocities from the test distribution in \cref{fig:wind_dist}, which is \emph{different} from that used for the training data~$\{\mathcal{T}_j\}_{j=1}^M$. In particular, the test distribution has a higher mode and larger support than the training distribution, thereby producing so-called out-of-distribution wind velocities at test time. In general, the robustness of meta-learned models and controllers to out-of-distribution tasks and train-test distribution shift is a core desideratum in meta-learning literature \citep{HospedalesAntoniouEtAl2021}.

\subsection{Results and Discussion}\label{sec:results-discussion}

We now present empirical test results for simulations of the PFAR and PVTOL systems subject to wind gust disturbances. For both systems, we compare the trajectory tracking performance of our meta-learned adaptive controller to the baseline methods outline in \cref{sec:baselines}. Our experiments are done in Python using NumPy \citep{HarrisMillmanEtAl2020} and JAX \citep{BradburyFrostigEtAl2018}. We use the explicit nature of Pseudo-Random Number Generation (PRNG) in JAX to set a seed prior to meta-training, and then methodically branch off the associated PRNG key as required throughout the train-test experiment pipeline. Thus, all of our results can be easily reproduced; code to do so is provided at~\url{https://github.com/StanfordASL/Adaptive-Control-Oriented-Meta-Learning}.

A benefit of our method is that it meta-learns model feature parameters $\varphi$, control gains (either $(\Lambda,K)$ for the PFAR system or $(c_x,c_y,k_x,k_y,k_\phi)$ for the PVTOL system), and the adaptation gain~$\Gamma$ offline. On the other hand, both the nominal feedback control and MRR baselines require gain tuning in practice by interacting with the real system. However, for the sake of comparison, we test every method with various combinations of feature parameters and control gains for each system on the same set of test trajectories.

\subsubsection{Testing on PFAR}
\begin{figure*}
    \centering
    \includegraphics[width=\textwidth]{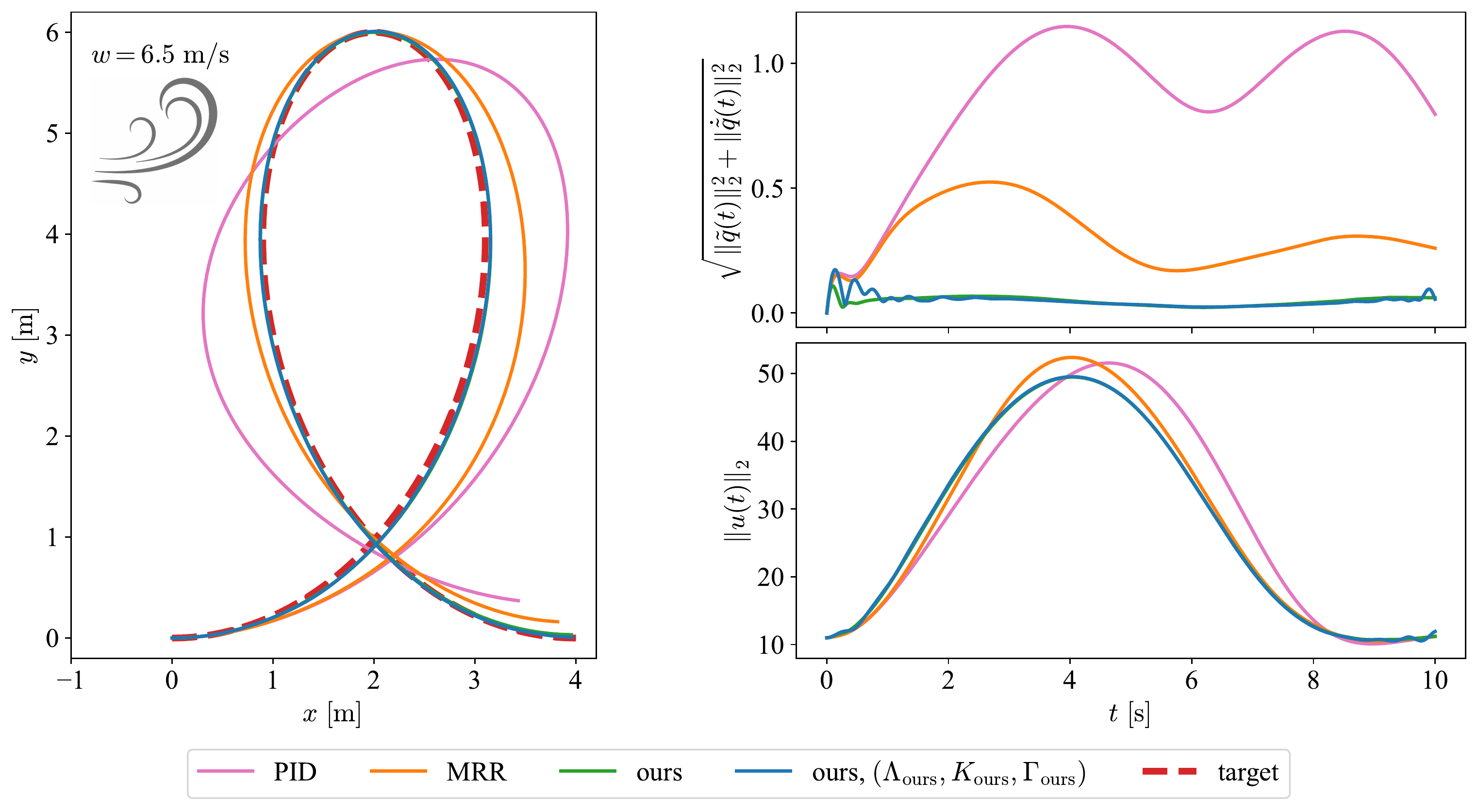}
    \caption{%
        Tracking results for the PFAR on a ``loop-the-loop'' with ${w = 6.5~\mathrm{m/s}}$, $M=10$, and $(\Lambda, K, \Gamma) = (I, 10I, 10I)$. We also apply our meta-learned features and gains $(\Lambda_\mathrm{ours}, K_\mathrm{ours}, \Gamma_\mathrm{ours})$. Every method expends similar control effort. However, with our meta-learned features, the tracking error ${\norm{x - x^*}_2 = \sqrt{\norm{\tilde{q}}_2^2 + \norm{\dot{\tilde{q}}}_2^2}}$ (where $x$ is overloaded to denote both position $x \in \R$ and state $x = (q,\dot{q})$) quickly decays after a short transient, while the effect of the wind pushing the vehicle to the right is more pronounced for the baselines.
    }\label{fig:results_traj}
\end{figure*}
\begin{figure*}
    \centering
    \includegraphics[width=\textwidth]{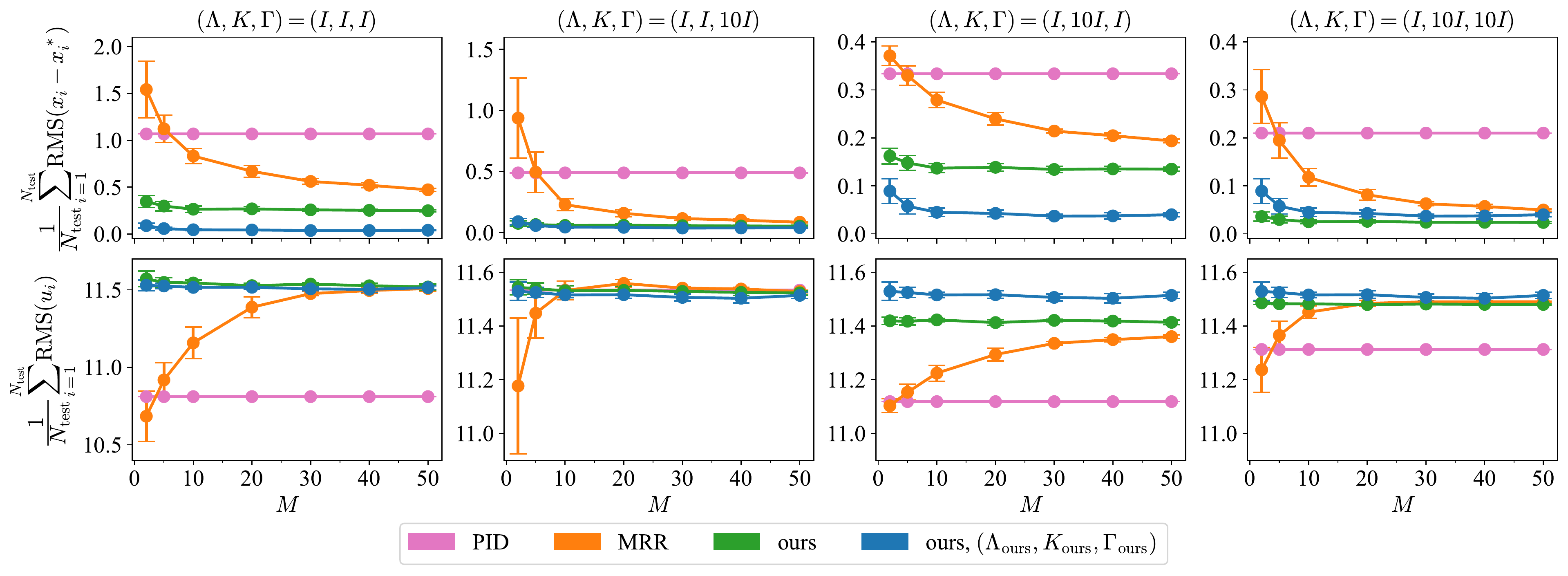}
    \caption{%
        Line plots of the average RMS tracking error $\frac{1}{N_\mathrm{test}}\sum_{i=1}^{N_\mathrm{test}}\mathrm{RMS}(x_i - x^*_i)$ and average control effort  $\frac{1}{N_\mathrm{test}}\sum_{i=1}^{N_\mathrm{test}}\mathrm{RMS}(u_i)$ for the PFAR system across $N_\mathrm{test}=200$ test trajectories versus the number of trajectories~$M \in \{2, 5, 10, 20, 30, 40, 50\}$ in the training data. For each method, we try out various control and adaptation gains $(\Lambda, K, \Gamma)$. With our method, we also use our meta-learned gains $(\Lambda_\mathrm{ours}, K_\mathrm{ours}, \Gamma_\mathrm{ours})$. Dots and error bars denote means and standard deviations, respectively, across 10~random seeds. The results for the PID controller do not vary since it does not require any training.
    }\label{fig:results}
\end{figure*}
For the PFAR system, we first provide a qualitative plot of tracking results for each method on a single ``loop-the-loop'' trajectory in \cref{fig:results_traj}, which clearly shows that our meta-learned features induce better tracking results than the baselines, while requiring a similar expenditure of control effort. For our method, the initial transient decays quickly, thereby demonstrating \emph{fast} adaptation and the potential to handle even time-varying disturbances.

For a more thorough analysis, we consider the Root-Mean-Squared (RMS) tracking error and control effort for each test trajectory~$\{x^*_i\}_{i=1}^{N_\mathrm{test}}$; for any vector-valued signal~$h(t)$ and sampling times~$\{t^{(k)}\}_{k=0}^N$, we define
\begin{equation}
    \mathrm{RMS}(h) \defn \sqrt{ \frac{1}{N} \sum_{k = 0}^N \norm{h(t^{(k)})}_2^2 }.
\end{equation}
We are interested in $\mathrm{RMS}(x_i - x^*_i)$ and $\mathrm{RMS}(u_i)$, where $x_i(t) = (q_i(t),\dot{q}_i(t))$ and $u_i(t)$ are the resultant state and control trajectories from tracking ${x^*_i(t) = (q^*_i(t),\dot{q}^*_i(t))}$. In \cref{fig:results}, we plot the averages ${ \frac{1}{N_\mathrm{test}}\sum_{i=1}^{N_\mathrm{test}}\mathrm{RMS}(x_i - x^*_i) }$ and $\frac{1}{N_\mathrm{test}}\sum_{i=1}^{N_\mathrm{test}}\mathrm{RMS}(u_i)$ across $N_\mathrm{test}=200$ test trajectories for each method. For our method and the MRR baseline, we vary the number of training trajectories~$M$, and thus the number of wind velocities from the training distribution in~\cref{fig:wind_dist} implicitly present in the training data. From~\cref{fig:results}, we observe the PID controller usually yields the highest tracking error, thereby indicating the utility of meta-learning features~$y(q,\dot{q};\varphi)$ to better compensate for~$f_\mathrm{ext}(q,\dot{q})$. We further observe in~\cref{fig:results} that, regardless of the control gains, using our features~$y(q,\dot{q};\varphi)$ in the adaptive controller~\cref{eq:slotine-li-controller-parametric} yields the lowest tracking error. Moreover, using our features with our meta-learned gains yields the lowest tracking error in all but one case. Our features sometimes induce a slightly higher control effort, especially when used with our meta-learned gains~$(\Lambda_\mathrm{ours}, K_\mathrm{ours}, \Gamma_\mathrm{ours})$. This is most likely since the controller can better match the disturbance~$f_\mathrm{ext}(q,\dot{q})$ with our features in closed-loop, and is an acceptable trade-off for improved tracking performance, which is our primary objective. In addition, when using our features with manually chosen or our meta-learned gains, the tracking error remains relatively constant over~$M$; conversely, the tracking error for the MRR baseline is higher for lower~$M$, and only reaches the performance of our method with certain gains for large~$M$. Overall, our results indicate:
\begin{itemize}
    \item The features~$y(q,\dot{q};\varphi)$ meta-learned by our control-oriented method are \emph{better conditioned for closed-loop tracking control} across a range of controller gains, particularly in the face of a distributional shift between training and test scenarios.
    \item The gains meta-learned by our method are competitive \emph{without manual tuning}, and thus can be deployed immediately or serve as a good initialization for further fine-tuning.
    \item Our control-oriented meta-learning method is \emph{sample-efficient} with respect to how much system variability is implicitly captured in the training data.
\end{itemize}
We again stress these comparisons could only be done by tuning the control gains for the baselines, which in practice would require interaction with the real system and hence further data collection. Thus, the fact that our control-oriented method can meta-learn good control gains \emph{offline} is a key advantage over regression-oriented meta-learning.

\subsubsection{Testing on PVTOL} Before testing and comparing our method to the baseline controller methods described in \cref{sec:baselines}, we must choose controller and adaptation gains for the baseline methods. Let us collectively notate the controller gains for the nominal differential-flatness-based controller~\cref{eq:pvtol-feedback} as
\begin{equation}
    \kappa \defn (c_x, c_y, k_x, k_y, k_\phi).
\end{equation}
For all tests on the PVTOL system, we compare the adaptive controller~\cref{eq:adaptive-control-affine-parametric} with our meta-learned gains to the following baseline configurations: 
\begin{enumerate}[label=(\Alph*)]
    \item the nominal controller with no adaptation and the initial controller gains $\kappa_\mathrm{init}$ used to collect training data;
    \item the nominal controller with no adaptation and the meta-learned controller gains $\kappa_\mathrm{ours}$ from our method;
    \item the adaptive controller~\cref{eq:adaptive-control-affine-parametric} with features learned via the MRR baseline, the initial controller gains $\kappa_\mathrm{init}$ used to collect training data, and the adaptation gain $\Gamma_\mathrm{init}$ from initialization of our meta-learned parameters; and
    \item the adaptive controller~\cref{eq:adaptive-control-affine-parametric} with features learned via the MRR baseline, and the meta-learned controller and adaptation gains $(\kappa_\mathrm{ours}, \Gamma_\mathrm{ours})$ from our method.
\end{enumerate}
Our goal with this set of configurations is an ablation study wherein the incremental benefits from using: 1) adaptation, 2) our meta-learned gains, and 3) our meta-learned features are demonstrated.

\begin{figure*}
    \centering
    \includegraphics[width=\textwidth]{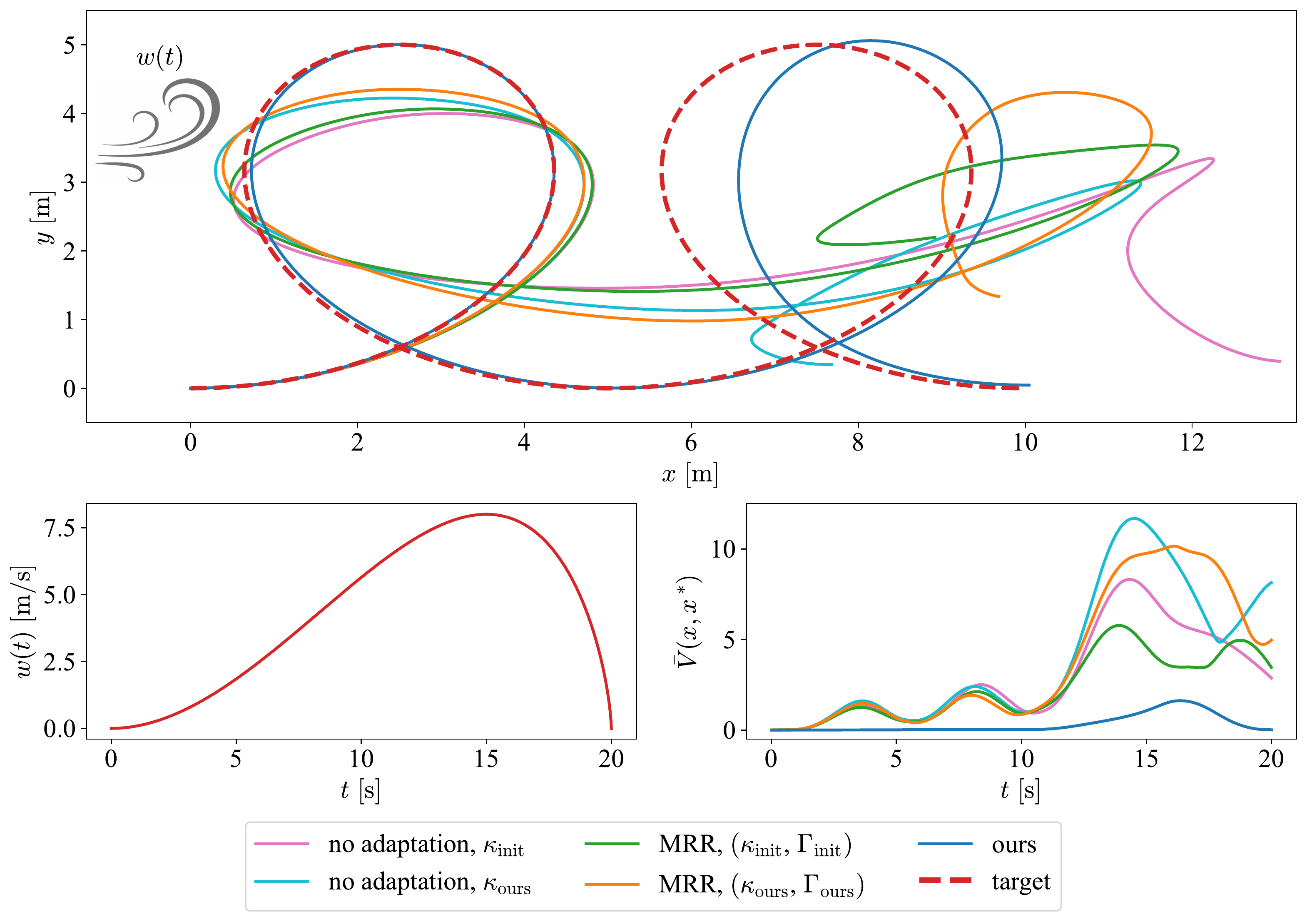}
    \caption{%
        Tracking results for the PVTOL system on a double ``loop-the-loop'' with $M = 10$ and a time-varying wind velocity $w(t)$. The certificate function~$\bar{V}$ from \cref{eq:pvtol-lyapunov} is also plotted to summarize the state tracking performance (where $x$ is overloaded to denote both position $x \in \R$ and vehicle state $x \in \R^6$). The value of~$\bar{V}$ for our meta-learned adaptive controller shows a ``bump'' at the time when the vehicle is buffeted away from the target trajectory by the wind.
    }\label{fig:pvtol_timevarying}
\end{figure*}
We begin by visualizing tracking results on a double ``loop-the-loop'' trajectory in \cref{fig:pvtol_timevarying}, with a \emph{time-varying} wind velocity~$w(t)$ that peaks at $8~\mathrm{m/s}$, which lies \emph{outside} of the training distribution in \cref{fig:wind_dist}. We also plot the certificate function~$\bar{V}$ from \cref{eq:pvtol-lyapunov} as a succinct scalar summary of the tracking performance for each method. From the first loop, we see immediately that all of the controllers except ours is significantly perturbed by even a small wind disturbance. As the wind velocity increases into the second loop, the vehicle using our meta-learned adaptive controller is marginally buffeted away from the target trajectory, but recovers as it exits the loop. Meanwhile, all of the other configurations suffer greatly from the high wind during the second loop.

\begin{figure}
    \centering
    \includegraphics[width=\columnwidth]{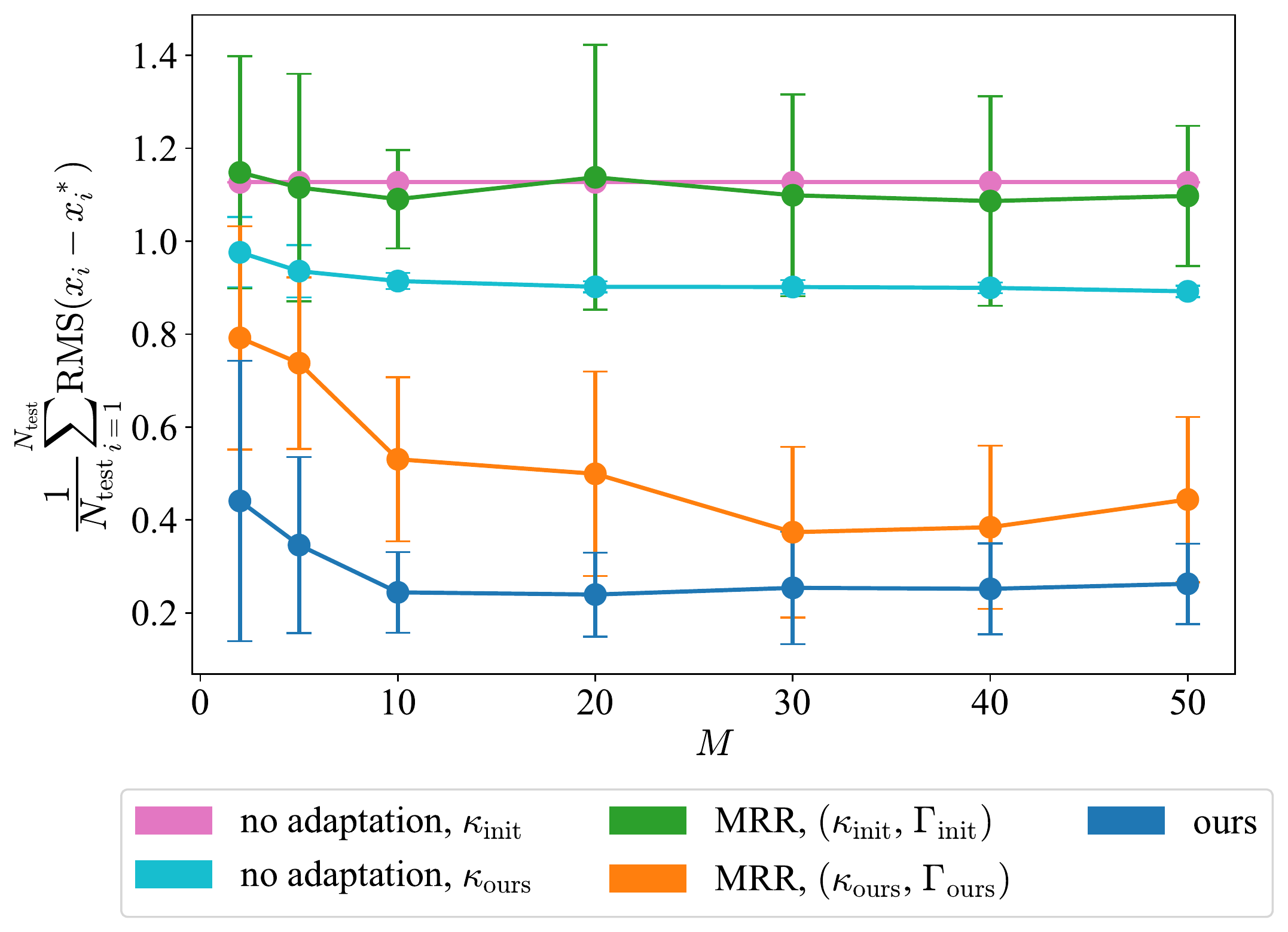}
    \caption{%
        Line plots of the average RMS tracking error $\frac{1}{N_\mathrm{test}}\sum_{i=1}^{N_\mathrm{test}}\mathrm{RMS}(x_i - x^*_i)$ for the PVTOL system over ${ N_\mathrm{test}=200 }$ test trajectories versus the number of trajectories ${ M \in \{2, 5, 10, 20, 30, 40, 50\} }$ in the training data. The results for the nominal feedback controller without any learned components do not vary since it does not require any training. Dots and error bars denote means and standard deviations, respectively, across 10~random seeds.
    }\label{fig:pvtol_lineplot}
\end{figure}
\begin{figure*}
    \centering
    \includegraphics[width=\textwidth]{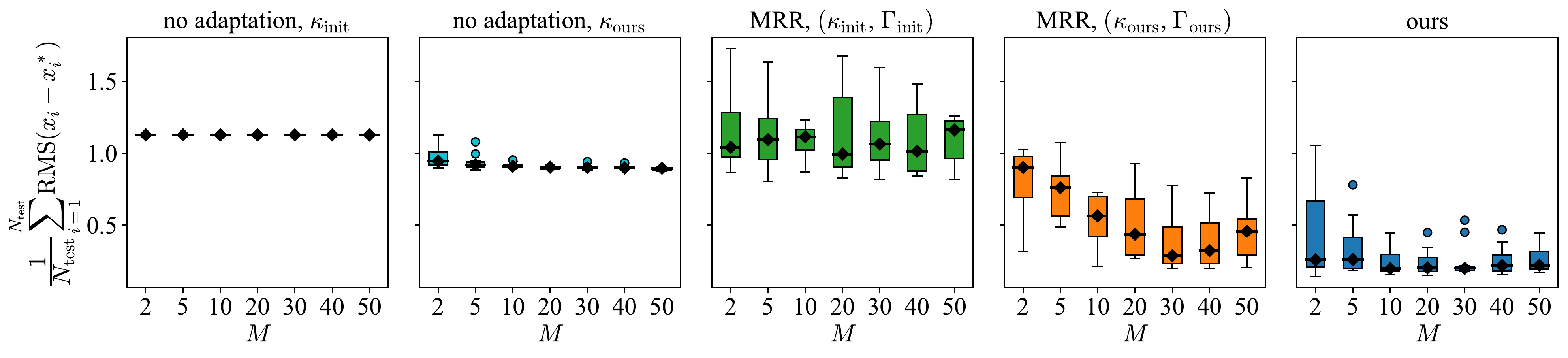}
    \caption{%
        Tukey box-and-whisker plots of the average RMS tracking error $\frac{1}{N_\mathrm{test}}\sum_{i=1}^{N_\mathrm{test}}\mathrm{RMS}(x_i - x^*_i)$ for the PVTOL system over $N_\mathrm{test}=200$ test trajectories versus the number of trajectories~$M \in \{2, 5, 10, 20, 30, 40, 50\}$ in the training data, for each controller configuration. The box plot statistics are computed across 10~random seeds. Medians are marked by black diamonds. The 25-th and 75-th percentiles are marked by the bottom and top edges, respectively, of each box. Each whisker extends an additional 1.5 times the interquartile range beyond the box edge. Outliers are marked by circles. The results for the nominal feedback controller without any learned components do not vary since it does not require any training.
    }\label{fig:pvtol_boxplot}
\end{figure*}
In a similar fashion to what we did for the PFAR system, we also analyze the average RMS tracking error ${ \frac{1}{N_\mathrm{test}}\sum_{i=1}^{N_\mathrm{test}}\mathrm{RMS}(x_i - x^*_i) }$ for each method across $200$~test trajectories, each alongside a fixed wind velocity~$w$ sampled from the test distribution in~\cref{fig:wind_dist}. For each configuration, we once again vary the number of training trajectories~$M$, and thus the number of wind velocities from the training distribution in~\cref{fig:wind_dist} implicitly present in the training data. \cref{fig:pvtol_lineplot} depicts the mean and standard deviation of ${ \frac{1}{N_\mathrm{test}}\sum_{i=1}^{N_\mathrm{test}}\mathrm{RMS}(x_i - x^*_i) }$ across 10~random seeds for each configuration. \cref{fig:pvtol_boxplot} displays more detail in the form of box plots for the average RMS tracking error across these seeds for each configuration. The spread in each plot is due to the effect the random seed has on sampling~$M$ trajectories, initializing parameters, and stochastic batch gradient descent during meta-training in our method and the MRR baseline. From \cref{fig:pvtol_lineplot} and \cref{fig:pvtol_boxplot}, beyond noting that our meta-learned adaptive controller outperforms all of the baselines, we make the following observations:
\begin{itemize}
    \item In comparing the non-adaptive controllers with $\kappa_\mathrm{init}$ and $\kappa_\mathrm{ours}$, we see that our meta-learned control gains are an improvement over $\kappa_\mathrm{init}$.
    
    \item In comparing the non-adaptive controllers to the adaptive controller using features meta-learned with MRR and $(\kappa_\mathrm{init},\Gamma_\mathrm{init})$, we see that MRR can learn features that lead to \emph{poor closed-loop performance} without gain tuning.
    
    \item In comparing the adaptive controller using features meta-learned with MRR and $(\kappa_\mathrm{ours},\Gamma_\mathrm{ours})$ to the other baseline configurations, we see that adaptation can lead to improved tracking performance with gain tuning.

    \item Finally, in comparing our meta-learned adaptive controller to the adaptive controller using features meta-learned with MRR and $(\kappa_\mathrm{ours},\Gamma_\mathrm{ours})$, we see once again that our meta-learned features are \emph{better conditioned for closed-loop tracking control}.
\end{itemize}
Overall, adaptation, tuned control and adaptation gains, and parametric model features individually have the potential to incrementally improve closed-loop performance. Critically, our meta-learning framework trains these components in an end-to-end fashion to realize the sum of these improvements.

\section{Conclusions \& Future Work}\label{sec:conclusion}

In this work, we formalized control-oriented meta-learning of adaptive controllers for nonlinear dynamical systems, offline from trajectory data. The procedure we presented is general and uses adaptive control as the base-learner to attune learning to the downstream control objective. We then specialized our procedure to fully-actuated Lagrangian and general control-affine dynamical systems, with adaptive controller designs parameterized by control gains, an adaptation gain, and nonlinear model features. We demonstrated that our control-oriented meta-learning method engenders better closed-loop tracking control performance at test time than when learning is done for the purpose of model regression.

There are a number of exciting future directions for this work. In particular, we are interested in control-oriented meta-learning with \emph{constraints}, such as for adaptive Model Predictive Control (MPC) \citep{AdetolaGuay2011,BujarbaruahZhangEtAl2018,SolopertoKohlerEtAl2019,KohlerKottingEtAl2020,SinhaHarrisonEtAl2022} with state and input constraints. Back-propagating through such a controller would leverage recent work on differentiable convex optimization \citep{AmosRodriguezEtAl2018,AgrawalBarrattEtAl2019,AgrawalBarrattEtAl2020}. We could also back-propagate through parameter constraints; for example, physical consistency of adapted inertial parameters can be enforced as Linear Matrix Inequality (LMI) constraints that reduce overfitting and improve parameter convergence \citep{WensingKimEtAl2017}. 

In addition, we want to extend our meta-learning framework in a principled manner to adaptive control for systems with unmatched uncertainties. Such uncertainties present a fundamental challenge for traditional adaptive controllers, since they cannot be cancelled stably by the control input \citep{LopezSlotine2021,SinhaHarrisonEtAl2022}. \citet{LopezSlotine2022} established a universal adaptation law using stability certificate functions, such as Lyapunov functions and Control Contraction Metrics (CCMs) \citep{ManchesterSlotine2017}, that are parameterized as a family of certificate functions corresponding to all possible values of the unmatched uncertainty, rather than just using a single certificate function corresponding to the nominal dynamics. To this end, we want to explore how meta-learning can be used to train such universal adaptive controllers defined in part by parametric stability certificates. This could build off of existing work on learning such certificates from data \citep{RichardsBerkenkampEtAl2018,SinghRichardsEtAl2020,BoffiTuEtAl2020,TsukamotoChungEtAl2021}.

\begin{acks}
We thank Masha Itkina for her invaluable feedback, and Matteo Zallio for his expertise in crafting~\cref{fig:pfar}. 
\end{acks}

\begin{funding}
This research was supported in part by the National Science Foundation (NSF) via Cyber-Physical Systems (CPS) award \#1931815 and Energy, Power, Control, and Networks (EPCN) award \#1809314, and the National Aeronautics and Space Administration (NASA) University Leadership Initiative via grant \#80NSSC20M0163. Spencer M.~Richards was also supported in part by the Natural Sciences and Engineering Research Council of Canada (NSERC). This article solely reflects the authors' own opinions and conclusions, and not those of any NSF, NASA, or NSERC entity.
\end{funding}

\begin{dci}
The authors declare that there is no conflict of interest.
\end{dci}

\bibliographystyle{plainnat}
\bibliography{ASL_papers,main}

\appendix
\crefalias{section}{appendix}

\section{Training Details}\label{app:training}

\subsection{Our Method}
Before meta-training, for both the PFAR and PVTOL systems we first train an ensemble of~$M$ models~$\{\hat{f}_j\}_{j=1}^M$, one for each trajectory $\mathcal{T}_j$, via gradient descent on the regression objective~\cref{eq:ensemble-training} with $\mu_\mathrm{ensem} = 10^{-4}$ and a single fourth-order Runge-Kutta step to approximate the integral over $t^{(k+1)}_{j} - t^{(k)}_j = 0.01~\mathrm{s}$. We set each~$\hat{f}_j$ as a feed-forward neural network with~$2$ hidden layers, each containing~$32$ $\tanh(\cdot)$ neurons. We perform a random $75\%/25\%$ split of the transition tuples in $\mathcal{T}_j$ into a training set~$\mathcal{T}_j^\mathrm{train}$ and validation set~$\mathcal{T}_j^\mathrm{valid}$, respectively. We do batch gradient descent via Adam \citep{KingmaBa2015} on~$\mathcal{T}_j^\mathrm{train}$ with a step size of~$10^{-2}$, over $1000$~epochs with a batch size of $\floor{0.25\abs{\mathcal{D}_j^\mathrm{train}}}$, while recording the regression loss with~$\mu_\mathrm{ensem} = 0$ on~$\mathcal{T}_j^\mathrm{valid}$. We proceed with the parameters for~$\hat{f}_j$ corresponding to the lowest recorded validation loss.

With the trained ensemble~$\{\hat{f}_j\}_{j=1}^M$, we can now meta-train $\theta_\mathrm{ours} \defn (\varphi, \Lambda, K, \Gamma)$ for the PFAR system or $\theta_\mathrm{ours} = (\varphi, c_x, c_y, k_x, k_y, k_\phi, \Gamma)$ for the PVTOL system. First, we randomly generate ${N = 10}$ smooth target trajectories~$\{x^*_i\}_{i=1}^N$ in the same manner described in \cref{sec:data-collection-training}. We then randomly sub-sample ${ N_\mathrm{train} = \floor{0.75 N} = 7 }$ target trajectories and $M_\mathrm{train} = \floor{0.75 M}$ models from the ensemble to form the meta-training set $\{(x^*_i,\hat{f}_j)\}_{i=1,j=1}^{N_\mathrm{train}, M_\mathrm{train}}$, while the remaining models and target trajectories form the meta-validation set. We set $\hat{A}y(x;\varphi)$ as a feed-forward neural network with~$2$ hidden layers of~$32$ $\tanh(\cdot)$ neurons each, where the adapted parameters~$\hat{A}(t) \in \R^{d \x 32}$ serve as the output layer. We set up the meta-problem~\cref{eq:meta-adaptive-control-ensemble} using $\{(x^*_i,\hat{f}_j)\}_{i=1,j=1}^{N_\mathrm{train}, M_\mathrm{train}}$, $\mu_\mathrm{ctrl} = 10^{-3}$, $\mu_\mathrm{meta} = 10^{-4}$, and either the adaptive controller~\cref{eq:slotine-li-controller-parametric} for the PFAR system or the adaptive controller~\cref{eq:adaptive-control-affine-parametric} for the PVTOL system. We then perform gradient descent via Adam with a step size of~$10^{-2}$ to train~$\theta_\mathrm{ours}$. We compute the integral in~\cref{eq:meta-adaptive-control-ensemble} via a fourth-order Runge-Kutta integration scheme with a fixed time step of $0.01~\mathrm{s}$. We back-propagate gradients through this computation in a manner similar to \citet{ZhuangDvornekEtAl2020}, rather than using the adjoint method for neural ODEs \citep{ChenRubanovaEtAl2018}, due to our observation that the backward pass is sensitive to any numerical error accumulated along the forward pass during closed-loop control simulations. We perform~$500$ gradient steps while recording the meta-loss from~\cref{eq:meta-adaptive-control-ensemble} with~$\mu_\mathrm{meta} = 0$ on the meta-validation set, and take the best meta-parameters~$\theta_\mathrm{ours}$ as those corresponding to the lowest recorded validation loss.

\subsection{MRR Baseline} 
To meta-train $\theta_\mathrm{MRR} \defn \varphi$, we first perform a random $75\%/25\%$ split of the transition tuples in each trajectory $\mathcal{T}_j$ to form a meta-training set $\{\mathcal{T}_j^\mathrm{meta\text{-}train}\}_{j=1}^M$ and a meta-validation set $\{\mathcal{T}_j^\mathrm{meta\text{-}valid}\}_{j=1}^M$. We set $\hat{A}y(x;\varphi)$ as a feed-forward neural network with~$2$ hidden layers of~$32$ $\tanh(\cdot)$ neurons each, where the adapted parameters~$\hat{A}(t) \in \R^{d \x 32}$ serve as the output layer. We construct the meta-problem~\cref{eq:meta-learning} using the task loss~\cref{eq:lstsq-task-loss}, the adaptation mechanism~\cref{eq:lstsq-adapt}, and $\mu_\mathrm{meta} = 10^{-4}$; for this, we use transition tuples from~$\mathcal{T}_j^\mathrm{meta\text{-}train}$. We then meta-train~$\theta_\mathrm{MRR}$ via gradient descent using Adam with a step-size of~$10^{-2}$ for $5000$ steps; at each step, we randomly sample a subset of $\floor{0.25\abs{\mathcal{T}_j^\mathrm{meta\text{-}train}}}$ tuples from $\mathcal{T}_j^\mathrm{meta\text{-}train}$ to use in the closed-form ridge regression solution. We also record the meta-loss with $\mu_\mathrm{meta} = 0$ on the meta-validation set at each step, and take the best meta-parameters~$\theta_\mathrm{MRR}$ as those corresponding to the lowest recorded validation loss.

\end{document}